\documentclass{article} 
\usepackage{iclr2027_conference,times}
\newcommand{\ourmethod}{\texttt{MetaHeta}}

\usepackage{amsmath,amsfonts,bm}

\def\eqref#1{equation~\ref{#1}}

\def\1{\bm{1}}

\DeclareMathAlphabet{\mathsfit}{\encodingdefault}{\sfdefault}{m}{sl}
\SetMathAlphabet{\mathsfit}{bold}{\encodingdefault}{\sfdefault}{bx}{n}

\usepackage[
    colorlinks=true,
    linkcolor=red,
    citecolor=teal,
    urlcolor=brown
]{hyperref}
\usepackage{url}
\usepackage[inline]{enumitem}
\usepackage{xcolor}         
\usepackage{amsmath}
\usepackage{amsthm}
\usepackage{amssymb}
\usepackage{graphicx}
\usepackage{subcaption}
\usepackage[linesnumbered,ruled,vlined]{algorithm2e}
\usepackage{booktabs}       
\usepackage{float}
\usepackage{wrapfig}
\usepackage[table]{xcolor}
\usepackage[framemethod=tikz]{mdframed}
\usepackage{tikz}
\DeclareRobustCommand{\thresholdline}{%
  \tikz[baseline=-0.5ex]{%
    \draw[gray, line width=1.2pt, dash pattern=on 4pt off 2pt]
      (0,0) -- (0.6,0);
  }%
}

\definecolor{foeglasspurple}{HTML}{f0fbff}

\newmdenv[
  backgroundcolor=foeglasspurple,
  linecolor=foeglasspurple,
  linewidth=0pt,
  roundcorner=10pt,
  innerleftmargin=8pt,
  innerrightmargin=8pt,
  innertopmargin=7pt,
  innerbottommargin=7pt,
  skipabove=6pt,
  skipbelow=6pt
]{contributionbox}

\newtheorem{proposition}{Proposition}

\newtheoremstyle{restated}
  {\topsep}      
  {\topsep}      
  {\itshape}     
  {}             
  {\bfseries}    
  {.}            
  {.5em}         
  {\thmname{#1}\thmnote{ #3}}  

\theoremstyle{restated}
\newtheorem*{propositionrestated}{Proposition}

\title{Few-Shot Bioactivity Prediction with \\Meta-Learning under Assay Heterogeneity}

\author{
Michal Kmicikiewicz\textsuperscript{1,2,4},
Tommy Rochussen\textsuperscript{2,3,4},
Vincent Fortuin\textsuperscript{3,4,6},
Ewa Szczurek\textsuperscript{1,5,7}
\AND
{\normalfont
\begin{tabular}{c}
\textsuperscript{1}Institute of AI for Health, Helmholtz Munich \\
\textsuperscript{2}School of Computation, Information and Technology, Technical University of Munich \\
\textsuperscript{3}Helmholtz AI, \textsuperscript{4}Munich Center for Machine Learning \\
\textsuperscript{5}Faculty of Mathematics, Informatics and Mechanics, University of Warsaw \\
\textsuperscript{6}Department of Computer Science \& Artificial Intelligence, \\
University of Technology Nuremberg \\
\textsuperscript{7}Faculty of Mathematics, Informatics and Statistics, LMU Munich \\
\texttt{\{michal.kmicikiewicz,ewa.szczurek\}@helmholtz-munich.de}
\end{tabular}
}
}

\iclrfinalcopy 
\begin{document}

\maketitle

\begin{abstract}
Accurate bioactivity prediction is a central challenge in early-stage drug discovery, as individual assays often contain too few measurements to train reliable models independently.
Meta-learning offers a principled approach to this few-shot setting, but assay heterogeneity may limit its effectiveness.
Here, we test this hypothesis and show that meta-learning performance degrades as meta-training tasks become more heterogeneous.
To address this, we introduce \ourmethod{}, a meta-learning framework that accounts for assay heterogeneity by conditioning predictions on auxiliary data from related assays, with relatedness defined flexibly from available assay information.
The architecture of \ourmethod{} combines linear attention over large auxiliary datasets with exact attention over scarce task-specific context, enabling efficient scaling to the former without compromising exact attention over the latter.
We demonstrate the benefits of our approach on assays from ChEMBL and BindingDB, improving few-shot bioactivity prediction and downstream compound prioritization in retrospective Bayesian optimization.

\end{abstract}

\section{Introduction}
Accurately predicting compound activity in biological assays is a central challenge in early-stage drug discovery. Reliable \textit{in silico} predictors can guide candidate selection and therefore reduce the time and resources required for experimental evaluation \citep{Jiménez-Luna02092021, Scalia2026}.
Machine learning models have been widely employed for this task, but their performance typically relies on the availability of large amounts of high-quality data \citep{Tian2024, 10.1039/c8sc00148k}.
However, since bioactivity measurements are expensive to acquire, individual assays often contain too little data to train accurate models in isolation, placing assay-specific bioactivity prediction in a few-shot learning regime \citep{Vella2022-jt, Altae-Tran2017-gs}.
A natural way to alleviate this data scarcity is to share information across assays, but naively pooling their measurements risks conflating distinct assay-specific relationships.

A particularly promising approach to this problem is meta-learning \citep{schmidhuber, hochreiter_meta}, which aims to learn shared structure across a collection of related \textit{tasks} to facilitate adaptation to new ones, but without directly pooling their measurements.
This collection of tasks is commonly referred to as the \textit{meta-dataset}.
In bioactivity prediction, each such task can naturally correspond to an individual assay, with a small number of observed compound–activity pairs forming a \textit{context set} from which activities are predicted for new compounds (\textit{target inputs}).

A variety of meta-learning approaches have been developed to enable information sharing across tasks.
\citet{pmlr-v70-finn17a} introduced MAML, a gradient-based algorithm that learns a shared parameter initialization from which the model can adapt to new tasks in a few gradient steps.
With a similar bi-level optimization objective, \citet{chen2023metalearning} learn a shared feature representation on top of which task-specific Gaussian Processes (GPs) are fit.
\begin{wrapfigure}[12]{r}{0.55\textwidth}
    \centering
    \includegraphics[width=0.55\textwidth]{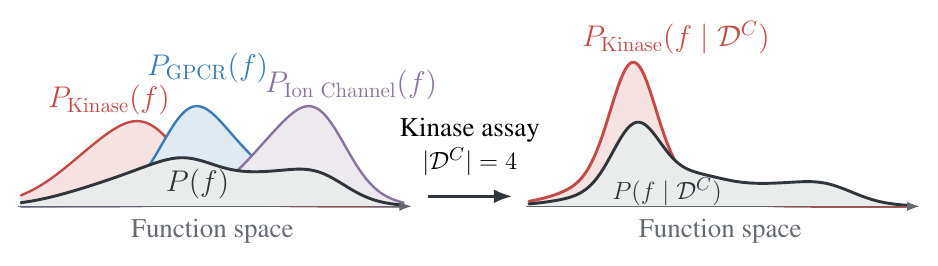}
    \caption{\textbf{Failure mode of meta-learning under assay heterogeneity.}
    A heterogeneous meta-dataset can induce a broad prior, such that limited task-specific data are insufficient to concentrate the posterior around functions relevant to a new assay.}
    \label{fig:example}
\end{wrapfigure}
Neural Processes (NPs) \citep{garnelo2018neuralprocesses, pmlr-v80-garnelo18a} instead amortize task adaptation by learning a mapping from a set of context points to predictive distributions at target locations, offering a scalable alternative to GPs. 
Transformer Neural Processes (TNPs) \citep{pmlr-v162-nguyen22b} parameterize this mapping using attention mechanisms, making them well suited to modeling sets while capturing rich interactions between context and target points. 
This architecture also underpins Prior-Data Fitted Networks (PFNs), 
which pre-train TNPs on massive synthetic datasets for off-the-shelf prediction on new tasks \citep{muller2022transformers}.

A key assumption underlying meta-learning is that tasks in the meta-dataset are sufficiently related to induce a shared prior that supports efficient adaptation to new tasks from limited data \citep{Baxter_2000}.
However, the heterogeneity of assays in large-scale bioactivity datasets, arising from differences in experimental conditions, endpoints, and measurement protocols \citep{Landrum2024}, can render this assumption questionable, potentially limiting the effectiveness of meta-learning.
We illustrate the intuition behind this potential failure mode in Figure~\ref{fig:example}.
Indeed, prior work suggests that meta-learning methods struggle to adapt to task-specific bioactivity data  \citep{stanley2021fsmol}, with \citet{schimunek2023contextenriched} reporting that even a naive baseline that ignores the context set can outperform several established meta-learning approaches. Here, we hypothesize that assay heterogeneity is a key factor limiting the effectiveness of meta-learning for bioactivity prediction, and  explicitly account for it in a dedicated modeling framework. 
\begin{contributionbox}
\textbf{Our contributions are as follows:}\\
\textbf{(i)}~we systematically study the effect of meta-training task
heterogeneity on few-shot bioactivity prediction, demonstrating in a controlled
setting that increasing heterogeneity can degrade meta-learning performance
across a range of methods and context sizes;\\
\textbf{(ii)}~motivated by these findings, we introduce \ourmethod{}, a meta-learning framework in which predictions are additionally conditioned on data from related assays.
This allows the model to approximate the posterior predictive under a more informative prior relevant to the target task, rather than under that induced by the entire heterogeneous meta-training distribution.
Assay relatedness can be defined flexibly based on target homology or known experimental conditions; \\
\textbf{(iii)}~to make conditioning on such auxiliary data computationally practical, we develop a TNP-based architecture that combines exact attention over scarce task-specific context with linear attention over potentially large auxiliary datasets, enabling efficient conditioning at scale.
\end{contributionbox}
We further demonstrate that findings from our controlled heterogeneity study translate into practical gains on assays from real-world bioactivity databases, improving few-shot bioactivity prediction and retrospective Bayesian optimization.

\section{Related work}
Early meta-learning approaches include metric-based methods \citep{matchingnets, prototypicalnets}, developed primarily for few-shot classification, and adaptation-based methods that learn shared structure across tasks to enable rapid task-specific adaptation from limited data \citep{pmlr-v70-finn17a, nichol2018firstordermetalearningalgorithms}. 
These ideas have also been applied to few-shot molecular property prediction \citep{nguyen2020metalearninggnninitializationslowresource, Altae-Tran2017-gs, schimunek2023contextenriched, NEURIPS2021_91bc333f, wang2024pintuning}, with much of this work focusing on classification.
More directly relevant to our probabilistic regression setting, deep kernel-based methods meta-learn shared feature representations while retaining task-specific Gaussian process inference \citep{chen2023metalearning, NEURIPS2020_b9cfe8b6}.
A complementary line of work instead amortizes task adaptation by directly learning a mapping from task-specific context and target inputs to predictive distributions, with Neural Processes (NPs) forming a canonical family of such methods \citep{garnelo2018neuralprocesses, pmlr-v80-garnelo18a, kim2018attentive, rochussen2026amortising}. 
In recent years, NPs equipped with transformer architectures \citep{NIPS2017_3f5ee243} (TNPs) have emerged as a prominent variant \citep{pmlr-v162-nguyen22b, feng2023latent, ashman_phd}, with Prior-Data Fitted Networks (PFNs) further extending this paradigm by pre-training TNPs on broad synthetic task distributions \citep{muller2022transformers}.
This line of work has been particularly influential in tabular learning \citep{Hollmann2025, grinsztajn2026tabpfn3technicalreport, qu2026tabiclv}, with recent studies showing that tabular foundation models can exhibit capabilities beyond conventional tabular settings \citep{nakerst2026tabularfoundationmodelsnontabular}, including molecular property prediction \citep{hicham2026tabularfoundationmodelsincontext}. 
We also discuss methods that explicitly account for task heterogeneity in Appendix~\ref{app:more_rel_work}.

\section{Setup}
Let $\mathcal{X} \subseteq \mathbb{R}^d$ denote the space of $d$-dimensional molecular representations and $\mathcal{Y} \subseteq \mathbb{R}$ the space of assay outcomes.
We model each assay $\tau$ as a task, corresponding to an unknown function $f_\tau: \mathcal{X} \rightarrow \mathcal{Y}$, where $f_\tau \sim P(f)$, and $P(f)$ denotes an underlying stochastic process on $\mathcal{X}$. 
For each $\tau$, we observe a finite dataset $\mathcal{D} = \{(x_n, y_n)\}_{n=1}^{N} \in \mathfrak D$
generated according to $y_n = f_\tau(x_{n}) + \epsilon_n$, with $\epsilon_n \sim \mathcal N(0,\sigma_\tau^2)$. Here, $\mathfrak{D} := \bigcup_{N\in\mathbb{N}} (\mathcal{X}\times\mathcal{Y})^N$ denotes the collection of all such datasets.
Each dataset can be partitioned into a context set $\mathcal D^C$ and a target set $\mathcal D^T$ of sizes $N_c := |\mathcal D^C|$ and $N_t := |\mathcal D^T|$, respectively.
We denote the corresponding inputs and outputs by $(X^C, Y^C) \in \mathbb{R}^{N_c \times d} \times \mathbb{R}^{N_c}$ and $(X^T, Y^T) \in \mathbb{R}^{N_t \times d} \times \mathbb{R}^{N_t}$. 
We assume access to a meta-dataset $\mathcal D_{\mathrm{meta}} := \{\mathcal D_\tau\}_{\tau=1}^M$ consisting of $M$ assay-specific datasets, each containing a potentially small number of observations.
Our goal is to meta-learn a predictor that maps an assay-specific context set and target inputs to a predictive distribution over target outcomes. 
Formally, we seek to learn the posterior predictive map $\pi:\mathfrak D\times\mathcal X^{N_t}\rightarrow\mathcal P(\mathcal Y^{N_t})$, which we parameterize by $p_\theta(Y^T \mid X^T, \mathcal{D}^C)$.

\section{\ourmethod{}}
\subsection{Beyond a single shared prior}
\paragraph{Failure mode of meta-learning under assay heterogeneity}
Consider a probabilistic meta-learner $p_\theta$ trained by maximizing the expected log-likelihood:
\begin{align}
\mathcal{L}(\theta)
&= \mathbb{E}_{X^T, Y^T, \mathcal{D}^C}
\left[
\log p_\theta\left(Y^T \mid X^T, \mathcal{D}^C\right)
\right] \\
&= -\mathbb{E}_{X^T, \mathcal{D}^C}
\left[
\mathrm{KL}\!\left[
p\left(Y^T \mid X^T, \mathcal{D}^C\right)
\,\|\, 
p_\theta\left(Y^T \mid X^T, \mathcal{D}^C\right)
\right]
\right] + \mathrm{const},
\end{align}
where the expectations over meta-training tasks and their context-target splits are estimated using $\mathcal D_{\mathrm{meta}}$, and $p(Y^T \mid X^T, \mathcal D^C)$ denotes the posterior predictive distribution (PPD) that $p_\theta(Y^T \mid X^T, \mathcal D^C)$ is trained to approximate \citep{pmlr-v202-nagler23a}. 
The PPD can be written as:
\begin{align}
p(Y^T \mid X^T,\mathcal D^C) = \int p(Y^T\mid X^T,f)P(df\mid\mathcal D^C),
\end{align}
where the prior $P(f)$ underlying the posterior $P(\cdot\mid\mathcal D^C)$ reflects the distribution of assay functions represented in $\mathcal {D}_{\mathrm{meta}}$.
Consequently, the composition of tasks in $\mathcal {D}_{\mathrm{meta}}$ shapes the shared prior $P(f)$ and, in turn, influences the ability of the meta-learner to adapt to new tasks given limited task-specific context $\mathcal{D}^C$.
Since bioactivity assays are often highly heterogeneous \citep{Landrum2024}, the prior induced by $\mathcal {D}_{\mathrm{meta}}$ may be too broad for the task-specific context to sufficiently concentrate the posterior on the relevant assay regime, limiting the effectiveness of meta-learning.

\paragraph{Accounting for heterogeneity through assay relatedness}
Motivated by this, we propose to model tasks as arising from a mixture of latent assay regimes $\xi_k \sim p(\xi)$, reflecting their shared characteristics such as the biological target or assay conditions.
Each such component $\xi_k$ induces its own prior $P_{\xi_k}(f)$, giving rise to a corresponding component-specific PPD $p_{\xi_k}(Y^T \mid X^T, \mathcal{D}^C)$.
The prior induced by the heterogeneous $\mathcal{D}_{\text{meta}}$ can then be viewed as a mixture of these component priors, with the corresponding PPD being a mixture of the component-specific PPDs. 
Since learning a separate posterior predictive map for each of potentially many latent components is impractical, we follow \citet{pmlr-v253-ashman24a} and instead use a single meta-learner $p_\theta(Y^T \mid X^T, \mathcal{D}^C, \mathcal{D}^{\text{aux}})$ to amortize prediction across different latent components. 
Here, $\mathcal D^{\mathrm{aux}}=\{\mathcal D_{\tau'}\}$ denotes a set of auxiliary datasets generated from the same latent component $\xi_k$ as the target task $\tau$:
\begin{equation}
    \label{eq_everything_everywhere_all_at_once}
    \xi_k \sim p(\xi), \quad f_\tau, \{f_{\tau'}\} \sim P_{\xi_k}(f), \quad \left(X^T, Y^T\right), \mathcal{D}^C \sim f_\tau, \quad \mathcal{D}_{\tau'} \sim f_{\tau'}.
\end{equation}
For a single target input-output pair $(x,y)$, \citet{pmlr-v253-ashman24a} show that:
\begin{equation}
\label{eq:icicl_main_theorem}
\begin{aligned}
\mathbb{E}_{\xi_k,\, \mathcal{D}^C,\, \mathcal{D}^{\text{aux}}}
\Big[
\mathrm{KL}\big(
p_{_{\xi_k}}(y \mid x, \mathcal{D}^C)
\;\|\;
p(y \mid x, \mathcal{D}^C, \mathcal{D}^{\text{aux}})
\big)
\Big]
\\
\leq\;
\mathbb{E}_{\xi_k,\, \mathcal{D}^C}
\Big[
\mathrm{KL}\big(
p_{_{\xi_k}}(y \mid x, \mathcal{D}^C)
\;\|\;
p(y \mid x, \mathcal{D}^C)
\big)
\Big].
\end{aligned}
\end{equation}
Intuitively, conditioning on auxiliary data provides additional evidence about the underlying assay regime $\xi_k$ and can yield a closer approximation to the corresponding PPD $p_{\xi_k}$ than conditioning on task-specific context alone.  
Such additional evidence can be particularly valuable when $\mathcal D^C$ is small and may not be sufficiently informative on its own.
This raises two practical challenges: identifying related assays to construct $\mathcal D^{\mathrm{aux}}$, and efficiently conditioning a meta-learner on both $\mathcal D^C$ and $\mathcal D^{\mathrm{aux}}$.

\subsection{Defining assay relatedness}
Depending on the modeling problem and available information, it might be useful to define assay relatedness at different levels. We consider two classes of settings.

\paragraph{Target-based relatedness} Consider a setting where $\xi_k$ is determined by target identity, such that assays sharing the same target are drawn from $P_{\xi_k}(f)$. Then, Eq.~\ref{eq:icicl_main_theorem} directly suggests constructing $\mathcal{D}^{\text{aux}}$ by collecting tasks with the same biological target as $\tau$. However, this construction can be overly restrictive with, for example, understudied protein targets, which are particularly relevant from a few-shot modeling perspective precisely because their assay data are often scarce \citep{Sharma2023-ii}. In such settings, known biological relationships between targets can be leveraged to construct $\mathcal{D}^{\text{aux}}$ instead, as formalised in Proposition~\ref{prop:ashman_to_protfams}:

\begin{proposition}
\label{prop:ashman_to_protfams}
Consider an extension of the data-generating process in Eq.~\ref{eq_everything_everywhere_all_at_once} in which the latent components are structured hierarchically as 
$\xi_{\mathrm{family}}\sim p(\xi)$ and $
\xi_{\mathrm{prot}}, \{\xi_{\mathrm{prot}'}\}
\sim
p(\cdot\mid\xi_{\mathrm{family}})$, with $f_\tau\sim P_{\xi_{\mathrm{prot}}}(f)$ and $\{f_{\tau'}\}\sim P_{\xi_{\mathrm{prot}'}}(f)$. Under this generative structure, assume that auxiliary data from target homologs provide no additional information for predicting $y$ once the target's protein-specific component is known. Then:
\[
\begin{aligned}
&\mathbb{E}_{\xi_{\mathrm{family}},\,\xi_{\mathrm{prot}},\,\mathcal{D}^C,\,\mathcal{D}^{\mathrm{aux}}}
\Bigg[
\mathrm{KL}\!\left(
p_{_{\xi_{\mathrm{prot}}}}(y \mid x,\mathcal{D}^C)
\;\Vert\;
p(y \mid x,\mathcal{D}^C,\mathcal{D}^{\mathrm{aux}})
\right)
\Bigg]
\\
&\qquad\leq
\mathbb{E}_{\xi_{\mathrm{family}},\,\xi_{\mathrm{prot}},\,\mathcal{D}^C}
\Bigg[
\mathrm{KL}\!\left(
p_{_{\xi_{\mathrm{prot}}}}(y \mid x,\mathcal{D}^C)
\;\Vert\;
p(y \mid x,\mathcal{D}^C)
\right)
\Bigg].
\end{aligned}
\]
\end{proposition}
Intuitively, assays of distinct but homologous proteins can still inform predictions for the target through their shared family-level component. The proof is in Appendix~\ref{app:proofs}.

\begin{figure}[t]
  \centering
  \includegraphics[width=1\linewidth]{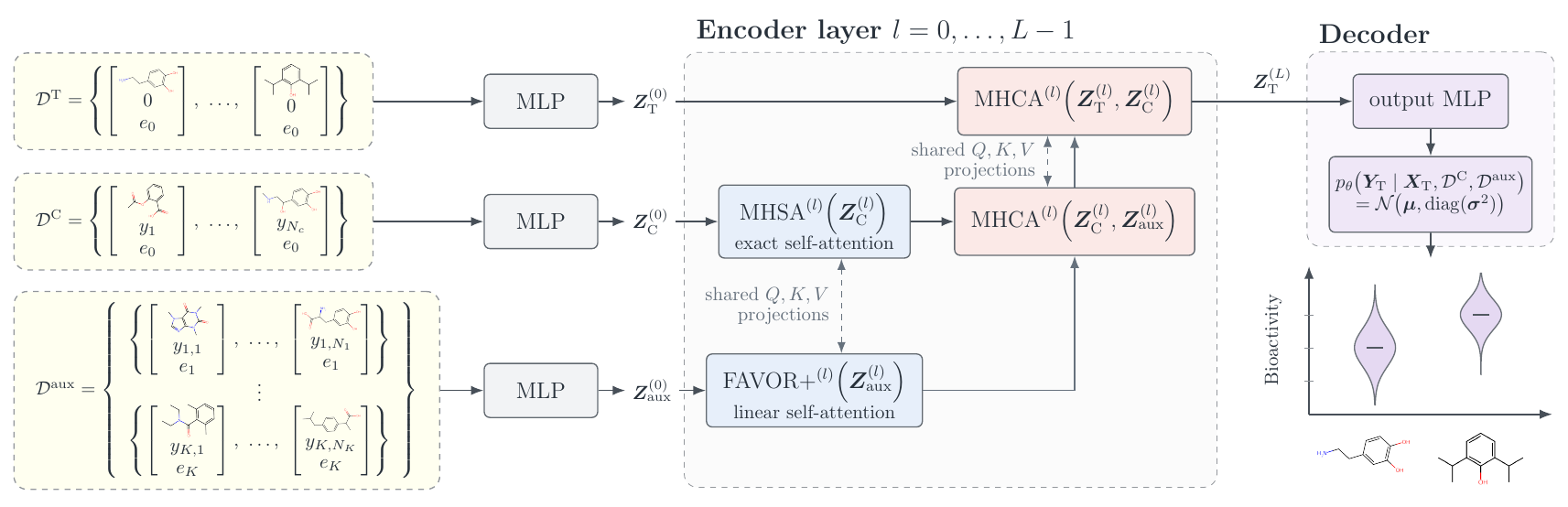}
    \caption{Architecture overview of \ourmethod. A shared input MLP embeds context, target, and concatenated auxiliary observations, augmented with task-of-origin indicators, with target activities masked to $0$. Encoder layers use exact self-attention for the context and FAVOR+ for auxiliary data, followed by context-to-auxiliary and target-to-context cross-attention. When $\mathcal D^{\mathrm{aux}}=\varnothing$, the auxiliary branch is skipped. Finally, an output MLP parameterizes a mean-field Gaussian predictive distribution. Formal definitions of the attention operators are given in Appendix~\ref{app:attention-operators}.}
  \label{fig:model}
\end{figure}

\paragraph{Description-based relatedness}
Target identities or known biological relationships may not always be available to practitioners. 
In such cases, we instead define relatedness based on broadly available information about assay protocols or experimental conditions. 
Specifically, for each $\tau$, we assume access to its natural-language description $a_\tau$, which may be noisy or incomplete. 
Prior to training, we map each $a_\tau$ to a $d$-dimensional dense representation $h_\tau = \phi(a_\tau) \in \mathbb{R}^d$, where $\phi$ denotes an embedding model; in our experiments $\phi$ is instantiated as PubMedBERT \citep{pubmedbert}. We then quantify relatedness between assays $\tau$ and $\tau'$ by cosine similarity between their description embeddings, $s_{\tau,\tau'} = \langle h_\tau,h_{\tau'}\rangle\bigl/(\|h_\tau\|\,\|h_{\tau'}\|\bigr)$. Since $s_{\tau,\tau'}$ provides only a heuristic for identifying useful auxiliary assays, we control retrieval selectivity using a threshold $\gamma$ and include $\mathcal{D}_{\tau'}$ in $\mathcal D^{\mathrm{aux}}$ if $s_{\tau,\tau'} \geq \gamma$.  
Meta-training with description-based relatedness is shown in Algorithm~\ref{algo}.

\subsection{Model}
The architecture is shaped by two competing design constraints arising from the variable availability of auxiliary data.
First, because $\mathcal{D}^{\mathrm{aux}}$ may frequently be empty, parameters dedicated exclusively to auxiliary data would often receive no learning signal, motivating parameter sharing between the processing of $\mathcal{D}^C$ and $\mathcal{D}^{\mathrm{aux}}$.
Second, $\mathcal{D}^{\mathrm{aux}}$ may instead contain many observations across multiple related datasets, making quadratic attention prohibitively expensive, while $\mathcal{D}^C$ is typically small.
Parameter sharing between $\mathcal D^C$ and $\mathcal D^{\mathrm{aux}}$ is therefore necessary, but their different scales call for different computational treatment.
Our architecture, shown in Figure~\ref{fig:model}, addresses this using two coupled TNP-like branches that share attention parameters while employing different self-attention mechanisms.
For both branches, each observation $[x, y, e_j]$ is first embedded with an MLP, with target observations having their activities masked to 0. Here, $e_j$ is a one-hot indicator of the task of origin, allowing observations from all auxiliary datasets to be concatenated into a single sequence without ambiguity. 
To preserve invariance to permutations of auxiliary datasets, the indicators are
assigned using a canonical ordering of $\mathcal D^{\mathrm{aux}}$, for example by task relatedness (Proposition~\ref{prop:perm_invariance} in Appendix~\ref{app:proofs}).
Within each encoder layer, context representations attend to the auxiliary representations, before target representations attend to the augmented context.
For the context branch, we employ standard softmax self-attention, whereas the auxiliary branch uses FAVOR+ \citep{choromanski2021rethinking}, a linear-time approximation based on positive orthogonal random features.
Since FAVOR+ provides an unbiased estimator of the softmax kernel, the same parameters can be used under both attention schemes while still targeting the same underlying operation.
This allows us to scale efficiently to large $\mathcal{D}^{\mathrm{aux}}$ without sacrificing exact attention over the task-specific context. The use of coupled TNP-like branches builds on ICICL-TNP \citep{pmlr-v253-ashman24a}; we provide a detailed comparison of the two architectures in Appendix~\ref{app:more_rel_work}.

\section{Experiments}
We structure the experiments as follows. In Section~\ref{sec:kinase_experiment}, we investigate the effect of meta-training task heterogeneity on meta-learning performance in a controlled environment.
In Section~\ref{exp:fs-mol}, we show the few-shot bioactivity prediction capabilities of \ourmethod{} on real-world bioactivity assays from ChEMBL \citep{gaulton2012chembl}.
In Section~\ref{exp:retro_bo}, we demonstrate its effectiveness for downstream compound prioritization through retrospective Bayesian optimization.
Finally, in Section~\ref{exp:ablation}, we ablate the auxiliary-conditioning mechanism and examine its computational cost.
All results are averaged over five random seeds and reported with standard errors.

\begin{figure}[t]
  \centering
\includegraphics[width=\linewidth]{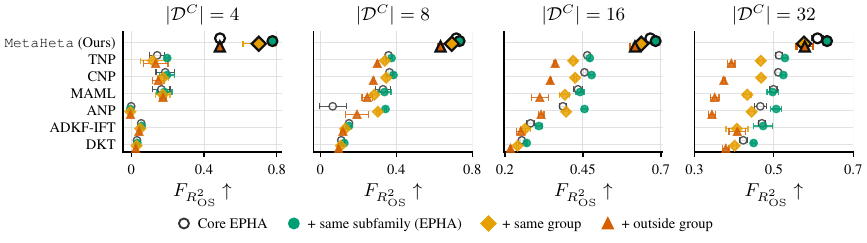}
  \caption{Effect of meta-training task heterogeneity on few-shot bioactivity prediction. Performance of all the considered methods generally deteriorates as $\mathcal D_{\mathrm{meta}}^{\mathrm{core}}$ is augmented with progressively less related tasks. \ourmethod{} remains the best-performing approach regardless of the training regime.}
  \label{fig:heterogeneity_exp}
\end{figure}

\subsection{Meta-learning under controlled assay heterogeneity}
\label{sec:kinase_experiment}
To study the effect of task heterogeneity in isolation, we construct a semi-synthetic benchmark based on kinase activity screening data from \citet{Drewry2017-pr}. 
Kinases provide a well-suited testbed: their common ATP-binding function is reflected in a highly conserved catalytic domain, while inhibitor response depends largely on variation at a small subset of active-site residues \citep{Sheridan2009-ua, Born2022-kk}. 
Coupled with the well-defined phylogenetic organization of the kinome, these properties provide a natural, biologically meaningful hierarchy of task relatedness.

\paragraph{Benchmark construction} We first fit a separate GP to each kinase's activity measurements and sample synthetic tasks from the resulting posteriors. 
We consider four meta-training regimes sharing a common meta-dataset $\mathcal D_{\text{meta}}^{\text{core}}$ of $5000$ tasks sampled from GPs corresponding to EPHA kinases, with the unaugmented $\mathcal D_{\text{meta}}^{\text{core}}$ serving as a baseline. 
The three additional regimes extend this common core, each with another $5000$ tasks sampled from GPs corresponding to: (i) EPHA kinases, (ii) non-EPHA tyrosine kinases (TK), and (iii) non-tyrosine kinases (non-TK), yielding progressively more heterogeneous task distributions, as reflected by relatedness measures reported in Table~\ref{tab:heterogeneity_desc} in Appendix~\ref{app:kinase_hetero}. 
Crucially, the $5000$ tasks added per regime differ only in the kinase-specific GP from which activities are sampled, so heterogeneity is induced solely through differences in the task-generating functions.
Evaluation is likewise fixed across regimes and performed on held-out validation and test sets, each comprising $250$ tasks sampled from GPs corresponding to EPHA kinases. A full description of the synthetic task-generation procedure is provided in Appendix~\ref{app:kinase_hetero}.
\begin{table}[b]
\centering
\caption{Few-shot prediction on EPHB test tasks after meta-training on $\mathcal{D}_{\mathrm{meta}}^{\mathrm{core}}$. \ourmethod{} matches or exceeds all baselines across both metrics and all context sizes.}
\label{tab:core_ephb}
\resizebox{1\textwidth}{!}{
\begin{tabular}{lcccccccc}
\toprule
& \multicolumn{2}{c}{Context size = 4}
& \multicolumn{2}{c}{Context size = 8}
& \multicolumn{2}{c}{Context size = 16}
& \multicolumn{2}{c}{Context size = 32} \\
\cmidrule(lr){2-3} \cmidrule(lr){4-5} \cmidrule(lr){6-7} \cmidrule(lr){8-9}
Method & F\textsubscript{NLL} $\boldsymbol{\uparrow}$ & $F_{R^2_{os}}$ $\boldsymbol{\uparrow}$ & F\textsubscript{NLL} $\boldsymbol{\uparrow}$ & $F_{R^2_{os}}$ $\boldsymbol{\uparrow}$ & F\textsubscript{NLL} $\boldsymbol{\uparrow}$ & $F_{R^2_{os}}$ $\boldsymbol{\uparrow}$ & F\textsubscript{NLL} $\boldsymbol{\uparrow}$ & $F_{R^2_{os}}$ $\boldsymbol{\uparrow}$ \\
\midrule
TNP
& $0.611 \pm 0.019$
& $-0.077 \pm 0.029$
& $0.827 \pm 0.004$
& $0.110 \pm 0.025$
& $0.900 \pm 0.002$
& $0.226 \pm 0.020$
& $0.931 \pm 0.001$
& $0.309 \pm 0.015$ \\
ANP
& $0.601 \pm 0.014$
& $-0.001 \pm 0.002$
& $0.812 \pm 0.004$
& $0.025 \pm 0.035$
& $0.899 \pm 0.001$
& $0.228 \pm 0.011$
& $0.931 \pm 0.001$
& $0.321 \pm 0.007$ \\
CNP
& $0.609 \pm 0.014$
& $0.124 \pm 0.010$
& $0.827 \pm 0.004$
& $0.156 \pm 0.004$
& $\mathbf{0.901 \pm 0.002}$
& $0.252 \pm 0.016$
& $0.930 \pm 0.001$
& $0.313 \pm 0.006$ \\
ADKF-IFT
& $0.262 \pm 0.027$
& $0.028 \pm 0.003$
& $0.760 \pm 0.007$
& $0.079 \pm 0.013$
& $0.882 \pm 0.002$
& $0.174 \pm 0.013$
& $0.925 \pm 0.002$
& $0.282 \pm 0.019$ \\
DKT
& $0.121 \pm 0.018$
& $0.021 \pm 0.002$
& $0.646 \pm 0.011$
& $0.071 \pm 0.002$
& $0.806 \pm 0.006$
& $0.158 \pm 0.012$
& $0.874 \pm 0.003$
& $0.285 \pm 0.010$ \\
MAML
& $-$
& $-0.101 \pm 0.119$
& $-$
& $0.108 \pm 0.045$
& $-$
& $0.257 \pm 0.013$
& $-$
& $0.314 \pm 0.023$ \\
\midrule
\ourmethod
& $\mathbf{0.829 \pm 0.015}$
& $\mathbf{0.501 \pm 0.018}$
& $\mathbf{0.882 \pm 0.004}$
& $\mathbf{0.346 \pm 0.053}$
& $\mathbf{0.901 \pm 0.004}$
& $\mathbf{0.327 \pm 0.084}$
& $\mathbf{0.934 \pm 0.003}$
& $\mathbf{0.333 \pm 0.041}$ \\
\bottomrule
\end{tabular}
}
\end{table}

\paragraph{Experimental setup}
We evaluate all methods in the few-shot regime at context sizes $|\mathcal{D}^C| \in \{4, 8, 16, 32\}$. 
We compare our approach against a set of strong probabilistic meta-learning methods, including members of the NP family, CNP, ANP, and TNP \citep{pmlr-v80-garnelo18a, kim2018attentive, pmlr-v162-nguyen22b}, as well as deep kernel-based methods, ADKF-IFT and DKT \citep{chen2023metalearning, NEURIPS2020_b9cfe8b6}. 
We additionally include MAML \citep{pmlr-v70-finn17a} as a widely used gradient-based baseline. Since the benchmark tasks are sampled from fitted GP posteriors, we can evaluate predictions relative to the corresponding true data-generating functions. 
We therefore report the fraction of oracle performance in terms of $R^2_{os}$ and NLL, denoted by $F_{R^2_{os}}$ and $F_{\mathrm{NLL}}$, respectively. For our method, $\mathcal{D}^{\mathrm{aux}}$ contains at most five same-target datasets, ordered by decreasing average $5$-nearest-neighbor Tanimoto similarity to the current task. Auxiliary datasets exceeding $512$ observations are randomly subsampled to that size. Full details are provided in Appendix~\ref{app:implementation-general}.

\paragraph{Assay heterogeneity can hurt meta-learning performance} The effect of increasing task heterogeneity is remarkably consistent across the considered methods and context sizes, as shown for $F_{R^2_{\mathrm{os}}}$ in Figure~\ref{fig:heterogeneity_exp}. 
Across all $28$ method--context combinations, augmenting $\mathcal D_{\mathrm{meta}}^{\mathrm{core}}$ with additional tasks from the same kinase subfamily improves $F_{R^2_{\mathrm{os}}}$ in 25 cases, whereas augmentation with tasks from other tyrosine and non-tyrosine kinases decreases it in 24 and 26 cases, respectively.
Moreover, the latter leads to a larger performance deterioration than the former in 23 comparisons.
The effects on $F_{\mathrm{NLL}}$ are smaller in magnitude but follow the same overall pattern, as shown in Figure~\ref{fig:heterogeneity_exp_nll} of Appendix~\ref{app:extended_results}. Same-subfamily augmentation improves performance in 20 of 24 method--context combinations; within-group and outside-group augmentations decrease it in 15 and 18 cases, respectively, and outside-group augmentation performs worse than within-group augmentation in 18 comparisons.
For both metrics, differences between regimes become more pronounced with increasing context size, as generally weaker performance at small context sizes can partially mask the effect of task heterogeneity. 
Overall, these results demonstrate that increasing meta-training task heterogeneity can impair meta-learning performance, despite the increase in available training data.

\begin{table}[t]
\centering
\caption{Performance comparison on FS-Mol. \ourmethod{} ranks best across seven of eight metric--context-size combinations, and ranks second on the remaining one.}
\label{tab:fs-mol}
\resizebox{1\textwidth}{!}{\begin{tabular}{lcccccccc}
\toprule
& \multicolumn{2}{c}{Context size = 4}
& \multicolumn{2}{c}{Context size = 8}
& \multicolumn{2}{c}{Context size = 16}
& \multicolumn{2}{c}{Context size = 32} \\
\cmidrule(lr){2-3} \cmidrule(lr){4-5} \cmidrule(lr){6-7} \cmidrule(lr){8-9}
Method & NLL $\boldsymbol{\downarrow}$ & $R^2_{os}$ $\boldsymbol{\uparrow}$ & NLL $\boldsymbol{\downarrow}$ & $R^2_{os}$ $\boldsymbol{\uparrow}$ & NLL $\boldsymbol{\downarrow}$ & $R^2_{os}$ $\boldsymbol{\uparrow}$ & NLL $\boldsymbol{\downarrow}$ & $R^2_{os}$ $\boldsymbol{\uparrow}$ \\
\midrule
TNP & 7.406 $\pm$ 2.519 & -0.070 $\pm$ 0.070 & 2.381 $\pm$ 0.216 & -0.048 $\pm$ 0.037 & 1.537 $\pm$ 0.026 & 0.153 $\pm$ 0.010 & 1.303 $\pm$ 0.028 & 0.199 $\pm$ 0.017 \\
ICICL-TNP
& 2.291 $\pm$ 0.112
& 0.069 $\pm$ 0.038
& 1.786 $\pm$ 0.022
& 0.117 $\pm$ 0.020
& 1.485 $\pm$ 0.007
& 0.155 $\pm$ 0.020
& 1.362 $\pm$ 0.044
& 0.131 $\pm$ 0.026 \\
ANP & 8.020 $\pm$ 4.795 & -0.048 $\pm$ 0.047 & 2.718 $\pm$ 0.265 & -0.071 $\pm$ 0.045 & 1.697 $\pm$ 0.075 & 0.056 $\pm$ 0.044 & 1.418 $\pm$ 0.098 & 0.131 $\pm$ 0.057 \\
CNP & 5.945 $\pm$ 2.766 & -0.176 $\pm$ 0.043 & 2.399 $\pm$ 0.229 & -0.013 $\pm$ 0.010 & 1.715 $\pm$ 0.028 & 0.063 $\pm$ 0.003 & 1.472 $\pm$ 0.029 & 0.120 $\pm$ 0.008 \\
ADKF-IFT & 22.605 $\pm$ 9.400 & 0.060 $\pm$ 0.001 & 4.509 $\pm$ 0.374 & 0.110 $\pm$ 0.005 & 2.235 $\pm$ 0.029 & 0.173 $\pm$ 0.003 & 1.582 $\pm$ 0.052 & $\mathbf{0.265 \pm 0.009}$ \\
DKT & 37.880 $\pm$ 18.236 & 0.022 $\pm$ 0.007 & 6.682 $\pm$ 0.564 & 0.068 $\pm$ 0.024 & 3.075 $\pm$ 0.066 & 0.148 $\pm$ 0.016 & 2.059 $\pm$ 0.070 & 0.224 $\pm$ 0.025 \\
TabICLv2 & - & -0.011 $\pm$ 0.006 & - & -0.008 $\pm$ 0.003 & - & -0.018 $\pm$ 0.003 & - & 0.010 $\pm$ 0.005 \\
TabPFNv3 & - & 0.026 $\pm$ 0.003 & - & 0.037 $\pm$ 0.003 & - & 0.073 $\pm$ 0.003 & - & 0.111 $\pm$ 0.001 \\
MAML & - & -0.214 $\pm$ 0.054 & - & -0.022 $\pm$ 0.025 & - & 0.050 $\pm$ 0.013 & - & 0.065 $\pm$ 0.007 \\
ST-MAML
& -
& -0.048 $\pm$ 0.005
& -
& 0.043 $\pm$ 0.005
& -
& 0.061 $\pm$ 0.003
& -
& 0.062 $\pm$ 0.006 \\
Frequent Hitters & - & 0.086 $\pm$ 0.007 & - & 0.095 $\pm$ 0.014 & - & 0.123 $\pm$ 0.010 & - & 0.131 $\pm$ 0.008 \\
Similarity Search & - & 0.012 $\pm$ 0.001 & - & 0.029 $\pm$ 0.001 & - & 0.041 $\pm$ 0.000 & - & 0.050 $\pm$ 0.000 \\
RF-$\mathcal{D}^{\mathrm{aux}}$
& -
& -0.228 $\pm$ 0.009
& -
& -0.196 $\pm$ 0.006
& -
& -0.194 $\pm$ 0.009
& -
& -0.213 $\pm$ 0.023 \\
\midrule
\textbf{\ourmethod{}} & $\mathbf{2.216 \pm 0.050}$ & $\mathbf{0.129 \pm 0.028}$ & $\mathbf{1.765 \pm 0.037}$ & $\mathbf{0.140 \pm 0.029}$ & $\mathbf{1.466 \pm 0.020}$ & $\mathbf{0.187 \pm 0.018}$ & $\mathbf{1.277 \pm 0.010}$ & 0.228 $\pm$ 0.011 \\
\rowcolor{gray!10}\ourmethod{} ($\gamma = 0.7$) & 2.325 $\pm$ 0.091 & 0.127 $\pm$ 0.026 & 1.778 $\pm$ 0.039 & 0.129 $\pm$ 0.026 & 1.476 $\pm$ 0.024 & 0.179 $\pm$ 0.016 & $\mathbf{1.277 \pm 0.012}$ & 0.223 $\pm$ 0.010 \\
\rowcolor{gray!10}\ourmethod{} ($\gamma = 0.8$) & 2.469 $\pm$ 0.124 & 0.119 $\pm$ 0.024 & 1.808 $\pm$ 0.036 & 0.112 $\pm$ 0.025 & 1.497 $\pm$ 0.029 & 0.170 $\pm$ 0.016 & 1.295 $\pm$ 0.025 & 0.218 $\pm$ 0.009 \\
\rowcolor{gray!10}\ourmethod{} ($\gamma = 0.9$) & 4.423 $\pm$ 1.938 & 0.098 $\pm$ 0.032 & 2.090 $\pm$ 0.163 & 0.074 $\pm$ 0.025 & 1.585 $\pm$ 0.051 & 0.151 $\pm$ 0.019 & 1.422 $\pm$ 0.076 & 0.204 $\pm$ 0.013 \\
\rowcolor{gray!10}\ourmethod{}-R & 2.438 $\pm$ 0.112 & -0.256 $\pm$ 0.164 & 1.959 $\pm$ 0.030 & -0.089 $\pm$ 0.076 & 1.606 $\pm$ 0.084 & 0.099 $\pm$ 0.045 & 1.363 $\pm$ 0.048 & 0.185 $\pm$ 0.013 \\
\bottomrule
\end{tabular}
}
\end{table}

\paragraph{Accounting for assay relatedness benefits few-shot meta-learning}
Across all heterogeneity regimes, \ourmethod{}, which explicitly accounts for task relatedness, achieves the best performance on both metrics at every context size. 
Notably, it retains this advantage even in a deliberately unfavorable cross-regime comparison: when trained under the most heterogeneous regime ($\mathcal D_{\mathrm{meta}}^{\mathrm{core}}$ augmented with non-TK tasks), it still outperforms baselines trained under the most favorable regime ($\mathcal D_{\mathrm{meta}}^{\mathrm{core}}$ augmented with additional EPHA tasks).
Even at the smallest context size $|\mathcal{D}^C| = 4$, \ourmethod{} already approaches oracle performance in terms of $F_{\mathrm{NLL}}$. 
This suggests that conditioning on data from related assays provides a valuable signal beyond the task-specific context, yielding informative predictive distributions even with very few observations from the target assay.

\paragraph{Auxiliary information transfers across homologous targets}
To examine cross-target transfer of auxiliary information, we retain the preceding setup with meta-training on $\mathcal D_{\mathrm{meta}}^{\mathrm{core}}$, but perform validation and testing on tasks from the EPHB kinase subfamily. Since these targets are not represented in $\mathcal D_{\mathrm{meta}}^{\mathrm{core}}$, no same-target assays are available. However, because the EPHA and EPHB subfamilies both belong to the EPH kinase family, we can use Proposition~\ref{prop:ashman_to_protfams} to construct $\mathcal D^{\mathrm{aux}}$ from the homologous EPHA tasks present in $\mathcal D_{\mathrm{meta}}^{\mathrm{core}}$ (see Appendix~\ref{app:implementation-ours} for full details).
The results in Table~\ref{tab:core_ephb} demonstrate that \ourmethod{} matches or exceeds all baselines across both metrics and all context sizes, with particularly large gains in the lowest-context regime.
This suggests that information from homologous targets can support prediction, particularly when task-specific data are limited.

\subsection{Meta-learning evaluation on real-world bioactivity assays}
\label{exp:fs-mol}

To assess whether the findings from the controlled assay-heterogeneity setting extend beyond the semi-synthetic benchmark, we evaluate \ourmethod{} on FS-Mol \citep{stanley2021fsmol}, a few-shot learning benchmark built from ChEMBL \citep{gaulton2012chembl} bioactivity assays.

\paragraph{Experimental setup} 
As in Section~\ref{sec:kinase_experiment}, we consider a highly data-scarce regime with $|\mathcal{D}^C| \in \{4, 8, 16, 32\}$. We construct $\mathcal{D}^{\mathrm{aux}}$ based on assay-description similarity with threshold $\gamma = 0.5$, selecting at most $5$ auxiliary datasets per task and subsampling each to at most 512 observations. 
We compare \ourmethod{} against previously introduced meta-learning baselines, together with ST-MAML \citep{pmlr-v180-wang22c}, a MAML variant designed to account for task heterogeneity, and ICICL-TNP \citep{pmlr-v253-ashman24a}, evaluated with the same $\mathcal{D}^{\mathrm{aux}}$ as our method.
We additionally include two state-of-the-art tabular foundation models, TabICLv2 \citep{qu2026tabiclv} and TabPFNv3 \citep{grinsztajn2026tabpfn3technicalreport}, as well as three non-meta-learning sanity checks: Frequent Hitters and Similarity Search \citep{schimunek2023contextenriched}, and RF-$\mathcal{D}^{\mathrm{aux}}$, 
a random forest fitted per target assay after combining its context with the observations in $\mathcal{D}^{\mathrm{aux}}$.
Since the underlying data-generating functions are unavailable, we report the standard $R^2_{os}$ and NLL where applicable. Details are in Appendix~\ref{app:implementation-general}.

\paragraph{Results}
Across the four context sizes and both metrics, our method achieves the best performance in $7$ out of $8$ comparisons. 
The gains are the largest in the most data-scarce regime and diminish as more task-specific context becomes available.
Consistent with the observations of \citet{schimunek2023contextenriched}, Frequent Hitters, which ignores the task context and instead learns from observations naively pooled across all assays in $\mathcal D_{\mathrm{meta}}$, outperforms several meta-learning approaches.
This suggests that standard meta-learners may struggle to exploit highly heterogeneous assay collections, particularly when only limited context data are available. 
In contrast, our approach is the only method that outperforms Frequent Hitters across all context sizes. 
Both tabular foundation models perform poorly, particularly at the smallest context sizes, providing further evidence that broad generic priors can be insufficient when limited context is available to identify the relevant task regime.
\ourmethod{} consistently outperforms ICICL-TNP, whose $R^2_{os}$ declines as the context size increases from $16$ to $32$ while ours improves, potentially reflecting the benefit of retaining exact self-attention over the context rather than processing it through pseudo-tokens.
Additionally, the poor performance of RF-$\mathcal{D}^{\mathrm{aux}}$ suggests that access to related auxiliary data alone is insufficient for effective prediction.

\begin{figure}
  \centering
\includegraphics[width=\linewidth]{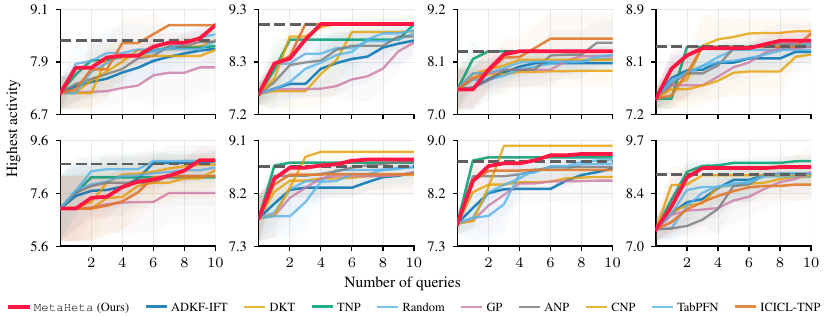}
    \caption{Retrospective Bayesian optimization trajectories on eight BindingDB assays. Dashed horizontal lines (\thresholdline) mark the activity threshold for the top 5\% of compounds in each assay; \ourmethod{} is the only method to identify such compounds in every assay.}
  \label{fig:bo}
\end{figure}

\subsection{Retrospective Bayesian optimization}
\label{exp:retro_bo}
To assess the practical utility of \ourmethod{}, we investigate whether its improved predictive performance in low-data settings translates into more effective compound prioritization when used as a surrogate model in retrospective Bayesian optimization on BindingDB assays \citep{10.1093/nar/gkae1075}.
\paragraph{Experimental setup}
We conduct the evaluation with initial conditions designed to resemble an early-stage compound-optimization scenario, in which activity measurements are available only for compounds belonging to a single, weakly active chemical series.
Specifically, for each of the eight considered assays and each random seed, we initialize optimization with four randomly selected compounds sharing a Bemis--Murcko scaffold \citep{10.1021/jm9602928} with median activity below the assay-specific 25th percentile.
We then perform 10 optimization rounds using the LogEI acquisition function \citep{ament2023unexpected}.
We compare \ourmethod{} against most of the probabilistic baselines introduced in Section~\ref{exp:fs-mol}, a task-specific GP with a Tanimoto kernel, and random selection. Further details on assay selection and candidate-pool construction are provided in Appendix~\ref{app:bo_details}.

\paragraph{Results}
As shown in Figure~\ref{fig:bo}, \ourmethod{} finds the most active compounds among the compared methods in three of the eight assays and ranks second in three others.
It is also the only approach to retrieve compounds that rank among the 5\% most active in each of the eight assays.
Table~\ref{tab:bo_rank} in Appendix~\ref{app:extended_results_bo} further demonstrates the strong and consistent performance of \ourmethod{} across all the considered assays: our approach achieves an average rank of $1.875$, leaving a substantial gap to the second-ranked method, TabPFN, with an average rank of $4.0$.
Notably, the task-specific GP ranks last in four of the eight assays, suggesting that fitting a surrogate independently from assay-specific observations alone can be unreliable in such a low-data regime.

\subsection{Ablation}
\label{exp:ablation}
In this section, we examine: (i) how the trade-off between predictive performance and inference time of \ourmethod{} depends on the number of auxiliary datasets and the choice of architecture; (ii) its sensitivity to the similarity threshold $\gamma$; and (iii) whether assay-description similarity provides a useful proxy for task relatedness.

\paragraph{Experimental setup}
We use the models trained on FS-Mol (Section~\ref{exp:fs-mol}). 
To answer (i), we compare predictive performance and inference time of \ourmethod{} and ICICL-TNP as we vary $|\mathcal{D}^{\mathrm{aux}}|$ at inference time, keeping $\gamma = 0$ and using the same auxiliary datasets for both methods.
For (ii), we evaluate stricter similarity thresholds by increasing $\gamma$ while retaining a maximum of five auxiliary datasets. 
Lastly, to address (iii), we evaluate \ourmethod{} with the default $\gamma = 0.5$ and a maximum of five auxiliary datasets, but shuffle the assignment of retrieved auxiliary assays to target assays. We denote this variant as \ourmethod{}-R.

\paragraph{Results}
Figure~\ref{fig:ablation} shows that conditioning on a single auxiliary dataset suffices to substantially improve $R^2_{os}$, with little further change as more datasets are added. Figure~\ref{fig:ablation_nll} in Appendix~\ref{app:extended_results_ablation} shows continued reductions in NLL, suggesting that additional datasets mainly refine uncertainty estimates rather than improve point predictions.
Although conditioning on a single auxiliary dataset increases inference time by approximately $2.3-2.5$ times, each additional dataset adds only around three seconds, resulting in $20\%$ overhead when expanding $\mathcal{D}^{\mathrm{aux}}$ from $1$ to $5$. 
We attribute this modest incremental cost to FAVOR+ attention, which allows our architecture to condition on five auxiliary datasets faster than ICICL-TNP does on a single one.
It also makes more effective use of the auxiliary signal, achieving higher mean $R^2_{os}$ across all context sizes and nonzero auxiliary-dataset counts.
Table~\ref{tab:fs-mol} shows that increasing $\gamma$ degrades performance, with the largest deterioration at $\gamma=0.9$, where many target assays lack even a single auxiliary assay satisfying the similarity criterion.
Notably, shuffling auxiliary-assay assignments in \ourmethod{}-R leads to comparatively modest increases in NLL but substantially degrades $R^2_{os}$, with the sharpest drops at small context sizes.
This provides strong evidence that assay-description similarity offers a useful basis for constructing $\mathcal D^{\mathrm{aux}}$ by capturing relatedness relevant to prediction, particularly when task-specific data are scarce.
\begin{figure}
  \centering
\includegraphics[width=\linewidth]{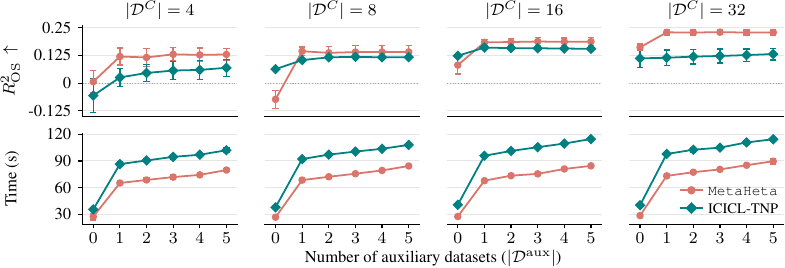}
  \caption{$R^2_{os}$ and inference time of \ourmethod{} and ICICL-TNP with varying cardinality of $\mathcal{D}^{\mathrm{aux}}$. Auxiliary conditioning improves predictive performance with modest incremental cost. Our architecture processes five auxiliary datasets faster than ICICL-TNP processes a single one.}
  \label{fig:ablation}
\end{figure}

\section{Conclusion}
In this work, we first investigate the effect of meta-training task heterogeneity on few-shot bioactivity prediction in a controlled setting, and show that increasing heterogeneity can systematically degrade meta-learning performance across methods and context sizes.
To address this, we introduce \ourmethod{}, a meta-learning framework which explicitly accounts for assay heterogeneity by conditioning predictions on auxiliary data from related assays. 
Its architecture enables this auxiliary conditioning to scale efficiently to large datasets while preserving exact attention over scarce task-specific context.
We further demonstrate the effectiveness of our method under three notions of assay relatedness: shared target identity, phylogenetic relationships between targets, and assay-description similarity.
Finally, we demonstrate the effectiveness of \ourmethod{} on bioactivity assays from ChEMBL and BindingDB, improving both few-shot prediction and downstream compound prioritization through retrospective Bayesian optimization.
We discuss the limitations of our approach and directions for future work in Appendix~\ref{app:discussion}.
Taken together, our results identify assay heterogeneity as an important challenge for few-shot bioactivity meta-learning and show that conditioning on related assays provides an effective and scalable way to address it.

\subsubsection*{Acknowledgments}
\begin{wrapfigure}[5]{r}{0.2\textwidth}
    \centering
    \includegraphics[width=\linewidth]{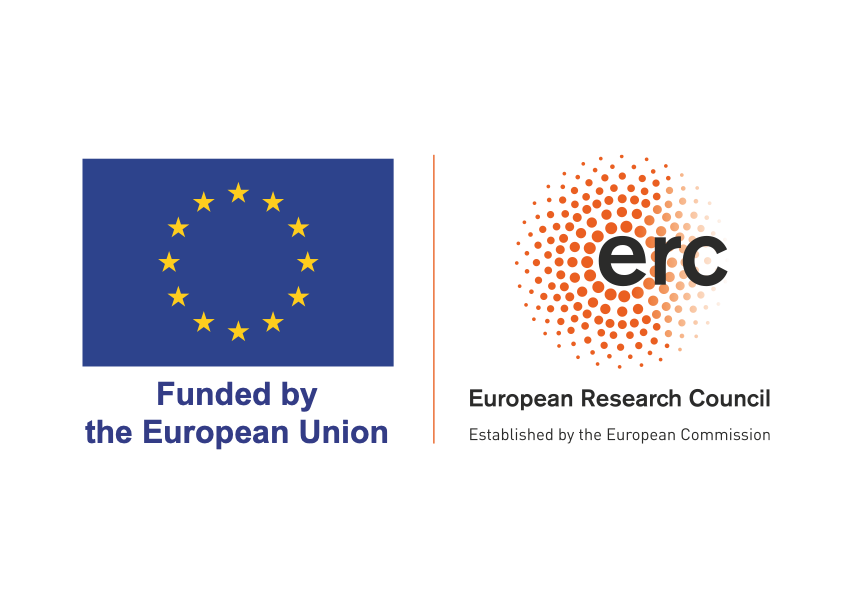}
\end{wrapfigure}
MK thanks Max from HPC for his invaluable contributions to maintaining the author’s caloric intake and, consequently, productivity.
VF was supported by the Branco Weiss Fellowship.
This project has received funding from the European Research Council (ERC) under the European Funding Union’s Horizon 2020 research and innovation programme (grant agreement No 101125506–DOG-AMP).

\bibliography{references}
\bibliographystyle{iclr2027_conference}

\newpage
\appendix
\section{Experimental details}
\label{app:general_experimental_details}

\subsection{Controlled-heterogeneity benchmark construction}
\label{app:kinase_hetero}

\paragraph{Data preparation}
We construct the semi-synthetic kinase benchmark from the PKIS2 profiling data of \citet{Drewry2017-pr}, comprising 645 compounds from 86 chemotypes screened against 392 wild-type human kinases. We first restrict the panel to the kinases considered in our benchmark and preprocess the molecular structures with RDKit \citep{landrum2006rdkit} by selecting the parent fragment, neutralizing charges, and canonicalizing tautomers.
Activity values are transformed from the unit interval to logit space after clipping extreme values. Molecules are represented using 1024-bit Morgan fingerprints with radius 2.

\paragraph{Kinase-specific GP fitting} 
For each kinase, we standardize the transformed activity values and fit an exact Gaussian process with a constant mean and a scaled Tanimoto kernel over Morgan fingerprints. This yields a separate posterior distribution over activity functions for each kinase.
We assess each kinase-specific GP using leave-one-chemotype-out cross-validation. In each fold, one chemotype is held out, the GP is refitted on the remaining compounds, and predictive performance is compared against a baseline that predicts the mean activity of the training fold. We retain only kinases with positive balanced leave-one-chemotype-out.

\paragraph{Molecule generation}
To obtain a sufficiently large molecular domain for synthetic task generation, we expand the original kinase-panel chemistry using CReM \citep{Polishchuk2020} mutation and growth operations. Generated molecules are filtered to satisfy basic drug-like property constraints, requiring molecular weight between \(150\) and \(650\), \(\log P\) between \(-1\) and \(7\), topological polar surface area below \(180\), at most \(6\) hydrogen-bond donors, \(12\) hydrogen-bond acceptors, and \(15\) rotatable bonds. We further retain only molecules whose maximum Tanimoto similarity to any molecule in the original panel lies between \(0.40\) and \(0.60\), yielding approximately \(10^6\) candidate molecules. These molecules are then partitioned into training, validation, and test molecular domains using an \(80/10/10\) Bemis--Murcko scaffold split, with scaffolds kept disjoint across the three partitions.

\paragraph{Task construction}
Each synthetic assay is generated by first sampling its size from a PERT distribution with minimum \(40\), mode \(100\), maximum \(1024\), and concentration \(20\). We then sample the corresponding number of molecules from the relevant molecular split and select a kinase-specific GP. Activities are generated by drawing jointly from the GP posterior over the selected molecules and adding Gaussian observation noise, with the fitted observation-noise standard deviation scaled by \(0.5\). 
Using this procedure, we first generate a core meta-dataset of \(5000\) tasks from GPs corresponding to EPHA kinases. We then generate an additional \(5000\) tasks for each heterogeneity regime while keeping the sampled task sizes and molecular inputs fixed across regimes. Only the kinase-specific GP used to generate the activity values is varied, with source GPs corresponding to (i) EPHA kinases, (ii) non-EPHA tyrosine kinases, or (iii) non-tyrosine kinases. We additionally generate \(250\) validation and \(250\) test assays from EPHA kinase GPs. For the homologous-target transfer experiment, we generate \(250\) validation and \(250\) test assays from EPHB kinase GPs, while keeping training restricted to EPHA tasks.

\paragraph{Kinase relatedness across regimes}
To quantify the increasing relatedness gap across the three kinase regimes, we consider two measures: similarity of kinase activity profiles across the PKIS2 compound panel and sequence divergence at the 29 active-site residues identified by \citet{Sheridan2009-ua}, with the resulting values reported in Table~\ref{tab:heterogeneity_desc}.
\begin{table}[H]
    \centering
    \caption{Quantitative characterization of kinase relatedness across heterogeneity regimes. $\rho_{\mathrm{act}}$: average pairwise Pearson correlation between kinase activity profiles over the PKIS2 compound panel; $d_{\mathrm{site}}$: average pairwise normalized Hamming distance across the 29 active-site residues identified by \citet{Sheridan2009-ua}.}
    \label{tab:heterogeneity_desc}
    \begin{tabular}{lcc}
        \toprule
        Regime & $\rho_{\mathrm{act}}$ & $d_{\mathrm{site}}$ \\
        \midrule
        EPHA & 0.466 & 0.149 \\
        TK & 0.316 & 0.367 \\
        EPHA + non-TK & 0.181 & 0.512 \\
        \bottomrule
    \end{tabular}
\end{table}

\subsection{Retrospective Bayesian optimization}
\label{app:bo_details}
\paragraph{Data preparation} We curate assays from BindingDB \citep{10.1093/nar/gkae1075} for retrospective Bayesian optimization experiments. Molecular structures are first processed with RDKit \citep{landrum2006rdkit} using the same procedure as described in Appendix~\ref{app:kinase_hetero}. We restrict the data to measurements against human, single-chain protein targets and remove target names associated with more than one UniProt \citep{10.1093/nar/gkae1010} identifier or more than one target-chain sequence, ensuring an unambiguous mapping between target name and protein identity. We further retain only BindingDB entries for which all measurements associated with a given DOI correspond to a single target. We consider affinity measurements reported as \(K_i\), \(\mathrm{IC}_{50}\), or \(K_d\), and retain only observations for which exactly one of these three endpoints is available, discarding observations with either no usable measurement or multiple reported endpoints. The remaining measurements are converted to a common \(p\)-activity scale as \(Y = 9-\log_{10}(x)\), where \(x\) is the reported affinity in nM. Finally, compounds with molecular weight above 900 Da are removed.

\paragraph{Assay selection} We select assays that are sufficiently large and chemically diverse to support retrospective compound optimization while retaining a non-degenerate activity landscape. Specifically, we require each assay to contain between 500 and 2000 compounds and at least 30 Bemis–Murcko scaffolds \citep{10.1021/jm9602928} represented by at least five compounds each. We further exclude assays for which any single activity value accounts for more than $25\%$ of the measurements. Finally, since the optimization experiments are commonly initialized from a single weakly active chemical series, we require each assay to contain at least five eligible starting scaffolds, defined as Bemis–Murcko scaffolds represented by at least five compounds whose median activity lies at or below the assay-specific 25th percentile. This ensures that each retained assay admits multiple chemically distinct low-activity starting series from which to simulate early-stage compound optimization. 
Since BindingDB contains records originating from both ChEMBL and the patent literature, we restrict assay selection to those derived from the latter to avoid potential overlap with the ChEMBL-based FS-Mol training data.
We conduct the experiments on eight assays satisfying all the above criteria with the following identifiers: \texttt{10.7270/Q2348QXR}, \texttt{10.7270/Q2N018M8}, \texttt{10.7270/Q2S46W6V}, \texttt{10.7270/Q2FN152J},
\texttt{10.7270/Q2BZ68G7},
\texttt{10.7270/Q2B27XNR},
\texttt{10.7270/Q2TB1B01}, and
\texttt{10.7270/Q2PV6PFT}.

\paragraph{Assay descriptions}
Assay descriptions were generated using GPT-5.6 Sol \citep{openai2026gpt56} by parsing the patent document corresponding to each of the eight BindingDB entries with the following prompt:
\begin{quote}
    Given the provided patent, write a single concise paragraph describing the experimental assay used to measure compound bioactivity. Include, when stated, the biological target and relevant construct or species, assay type and biological context (e.g., biochemical, cellular, binding, functional), key reagents or substrates, experimental conditions such as concentrations, incubation times, temperature, buffer or medium, the measurement technology or equipment, and the measured endpoint or readout. Describe only information explicitly supported by the source and do not infer missing experimental details. Do not report or summarize compound-specific activity values, potency measurements, active/inactive classifications, hit rates, structure–activity relationships, rankings, or conclusions about which compounds performed well or poorly. The description should characterize how the assay was performed and what it measured, independently of the observed compound outcomes. Keep the output to one short paragraph.
\end{quote}
The resulting descriptions used in the experiment were:
\begin{itemize}

    \item \texttt{10.7270/Q2348QXR}: A biochemical ATPase assay measured ATP hydrolysis by recombinant human DHX9 in the presence of an annealed double-stranded RNA substrate. The 10-\(\mu\)L reactions used DHX9 delivered at 1.25 nM in 40 mM HEPES, pH 7.5, 0.01\% Tween-20, 0.01\% BSA, 1 mM DTT, 5 mM MgCl$_2$, and RNase inhibitor and were initiated with 30 nM double-stranded RNA and 10 \(\mu\)M ATP. After one hour at room temperature, ADP-Glo reagent was added for 40 minutes, followed by kinase-detection reagent for another 40 minutes. Luminescence was recorded on an EnVision plate reader, providing an ADP-production readout of DHX9 ATPase activity.

    \item \texttt{10.7270/Q2N018M8}: Endothelial and hepatic lipase activity was measured biochemically using the fluorescent FRET substrate A10070 incorporated into DMPG vesicles. Endothelial lipase was supplied in conditioned medium from HT-1080 cells engineered to overexpress endogenous EL, while hepatic lipase was partially purified from conditioned medium of HL-expressing COS cells. Enzyme was preincubated with compound for 30 minutes at 37\(^\circ\)C in 50 mM HEPES, pH 8.0, 50 mM NaCl, and 1 mM CaCl$_2$, after which diluted substrate vesicles initiated the reaction. Fluorescence was recorded for 60 minutes on a Gemini plate reader at 490-nm excitation and 530-nm emission, and the slope between 300 and 900 seconds represented the reaction rate.

    \item \texttt{10.7270/Q2S46W6V}: A fluorescence-polarization competition assay measured binding of recombinant eIF4E to the fluorescent cap analogue EDA-m$^7$GDP-ATTO-550. Final reactions contained 75 nM eIF4E, 20 nM fluorescent ligand, and 1\% DMSO in 20 mM Tris, pH 7.5, 100 mM NaCl, 1 mM DTT, and 0.02\% Tween-20. Protein and ligand were preincubated for five minutes before compound addition, and 100-\(\mu\)L reactions in black half-area 96-well plates were equilibrated for 30 minutes at room temperature while protected from light. A PerkinElmer Victor 2 multi-label counter measured fluorescence polarization as the competitive-binding readout.

    \item \texttt{10.7270/Q2FN152J}: An isolated-enzyme HTRF kinase assay measured phosphorylation by the recombinant His-tagged cytoplasmic domain of TrkA/NTRK1, residues 441--796, using the Cisbio KinEASE-TK kit. The assay used 4 nM TrkA, 40 \(\mu\)M ATP, and kit buffer supplemented with 10 mM MgCl$_2$ and 2 mM DTT; enzyme and ATP were preactivated for 20 minutes at room temperature, exposed to compound for 20 minutes, and then incubated with 1 \(\mu\)M biotinylated tyrosine-kinase substrate for 35 minutes. Detection buffer containing streptavidin-XL665 and a cryptate-labelled phosphotyrosine antibody stopped the reaction, followed by 60 minutes at room temperature and EnVision measurement of the 665/620-nm emission ratio.

    \item \texttt{10.7270/Q2BZ68G7}: A biochemical radiometric methyltransferase assay measured incorporation of tritium from $^3$H-SAM into GST-tagged MEKK2 by purified full-length human SMYD3. Compounds and 0.4 nM SMYD3 were preincubated for 30 minutes at room temperature in 25 mM Tris-Cl, pH 8.0, 1 mM TCEP, 0.005\% bovine skin gelatin, and 0.005\% Tween-20; 8 nM $^3$H-SAM and 12 nM MEKK2 then initiated the reaction in 2\% DMSO. Non-radiolabelled SAM stopped incorporation, and radiolabelled MEKK2 was precipitated onto scintillation-proximity beads in 0.5 M citric acid. A PerkinElmer TopCount reader measured incorporated radioactivity as disintegrations or counts per minute.

    \item \texttt{10.7270/Q2B27XNR}: A biochemical time-resolved fluorescence assay measured ERK2-catalysed phosphorylation of a truncated ATF2 substrate fused to GFP. Reactions contained 0.25 nM ERK2, 0.4 mM ATF2-GFP, 20 \(\mu\)M ATP, 50 mM Tris at pH 7.5, 10 mM MgCl$_2$, 1 mM EGTA, 0.01\% Triton X-100, 1 mM DTT, and 2.5\% DMSO and proceeded for 30 minutes at room temperature. Reactions were stopped with TR-FRET dilution buffer containing 25 mM EDTA and 2 nM terbium-labelled anti-phospho-ATF2 (Thr71), incubated for at least another 30 minutes, and read on a Pherastar at 340-nm excitation and 520/495-nm emission; the emission ratio was the phosphorylation readout.

    \item \texttt{10.7270/Q2TB1B01}: A biochemical time-resolved fluorescence assay measured ERK2 phosphorylation of truncated ATF2-GFP. Reactions contained 0.25 nM ERK2, 0.4 \(\mu\)M ATF2-GFP, 20 \(\mu\)M ATP, 50 mM Tris at pH 7.5, 10 mM ``MgC$_2$'' as printed in the patent, 1 mM EGTA, 0.01\% Triton X-100, 1 mM DTT, and 2.5\% DMSO and were incubated for 30 minutes at room temperature. TR-FRET dilution buffer, 25 mM EDTA, and 2 nM terbium-labelled anti-phospho-ATF2 (Thr71) stopped the reaction; after at least 30 additional minutes, a Pherastar measured fluorescence at 520 and 495 nm following 340-nm excitation, and the emission ratio represented phosphorylation.

    \item \texttt{10.7270/Q2PV6PFT}: A low-ATP biochemical kinase assay used recombinant full-length N-terminally His-tagged human TBK1 expressed in insect cells and a biotinylated peptide substrate, biotin-Ahx-GDEDFSSFAEPG-amide. TBK1 and compound were preincubated for 15 minutes at 22\(^\circ\)C in HEPES buffer containing MgCl$_2$, DTT, BSA, Nonidet-P40, and protease inhibitors; phosphorylation then proceeded for 30 minutes at 22\(^\circ\)C with 10 \(\mu\)M ATP and 1 \(\mu\)M peptide. EDTA stopped the reaction, and streptavidin-XL665, an anti-phosphoserine antibody, and europium-labelled anti-mouse IgG were added for one hour. A Pherastar or Viewlux TR-FRET reader measured 620- and 665-nm emissions after 350-nm excitation, with their ratio representing phosphorylated substrate.

\end{itemize}

\subsection{Metrics}
\label{app:fs-mol-metrics}
We evaluate all methods using negative log-likelihood (NLL), where applicable, and the out-of-sample coefficient of determination \(R^2_{\mathrm{os}}\), following \citet{chen2023metalearning}. For a task with context set \(\mathcal{D}^C=(X^C,Y^C)\) and target set \(\mathcal{D}^T=(X^T,Y^T)\), we compute
\[
R^2_{\mathrm{os}}
=
1-
\frac{
    \sum_{n=1}^{N_t}
    \left(y_n^T-\hat y_n^T\right)^2
}{
    \sum_{n=1}^{N_t}
    \left(y_n^T-\bar y^C\right)^2
},
\qquad
\bar y^C
=
\frac{1}{N_c}\sum_{n=1}^{N_c}y_n^C,
\]
where \(\hat y_n^T\) denotes the predictive mean at target input \(x_n^T\). 

Since the synthetic assays in Section~\ref{sec:kinase_experiment} are sampled from known kinase-specific GP posteriors, we have access to their corresponding data-generating functions. This allows us to compute an exact GP oracle by conditioning the generating GP on the same context set \(\mathcal{D}^C\) available to each method. The oracle uses the same observation-noise model as the synthetic data-generation process.
We report performance relative to this oracle. We further define:
\begin{align}    
F_{R^2_{\mathrm{os}}}
=
\frac{
R^2_{\mathrm{os}}
}{
R^{2,\mathrm{oracle}}_{\mathrm{os}}
}.
\end{align}
For NLL, we instead compare the corresponding average predictive densities:
\begin{align}    
F_{\mathrm{NLL}}
=
\frac{
\exp(-\mathrm{NLL})
}{
\exp(-\mathrm{NLL}_{\mathrm{oracle}})
}
=
\exp\!\left(
\mathrm{NLL}_{\mathrm{oracle}}-\mathrm{NLL}
\right).
\end{align}
Since \(\exp(-\mathrm{NLL})\) is the geometric mean predictive density, \(F_{\mathrm{NLL}}\) has a direct interpretation as the ratio between the geometric mean predictive density of a method and that of the oracle. In both cases, a value of \(1\) corresponds to oracle-level performance.

\subsection{Implementation details}
\label{app:implementation-general}
We evaluate all methods at context sizes \(|\mathcal{D}^C|\in\{4,8,16,32\}\). We largely follow \citet{stanley2021fsmol}, but deviate in how context size is handled during training and validation. Specifically, for each context size, we train and validate a separate model using that same fixed context size, and evaluate it only at the corresponding context size. Given the highly data-scarce regime considered here, this ensures that every method is optimized for the same task-specific data budget on which it is evaluated, and avoids conflating few-shot performance with generalization across context sizes. This also places amortized methods, which can naturally accommodate varying context sizes, on the same footing as prior FS-Mol evaluations in which methods are trained using a fixed support-set size \citep{stanley2021fsmol, chen2023metalearning}. For each evaluation assay, we sample \(10\) context--target splits and average evaluation metrics across splits before averaging across assays. 
For the controlled-heterogeneity experiments, all methods are given the same molecular representation used by the task-generating GPs, namely 1024-bit radius-2 Morgan fingerprints. This ensures that differences in performance relative to the GP oracle cannot be attributed to a mismatch in molecular representation. 
For the experiments on FS-Mol, we do not enforce a single representation across methods: we align representations where appropriate, while retaining method-specific feature extractors when they are part of the original approach or using alternative representations when these perform better. This is because imposing a single shared representation may disadvantage particular baselines whose performance is strongly tied to the choice of molecular representation.
For the retrospective Bayesian optimization experiments, we train models with the context size sampled uniformly from \(\{4,\ldots,32\}\). During validation, we evaluate each model on three independently sampled context--target splits at each context size \(|\mathcal{D}^C|\in\{4,8,16,32\}\), and average performance across all splits and context sizes.

\subsubsection{\ourmethod{}}
\label{app:implementation-ours}
\paragraph{Controlled-heterogeneity experiment}
Our model is trained with meta-batch size \(32\) for \(10\) epochs using AdamW \citep{loshchilov2018decoupled} with learning rate \(10^{-4}\), weight decay \(0.01\) and a cosine learning-rate schedule with \(10\%\) warm-up. 
We evaluate on the validation set after each epoch and retain the checkpoint with the lowest validation NLL.
We use at most \(5\) auxiliary datasets per task and subsample each auxiliary dataset to \(512\) observations. The input embedder consists of \(2\) layers with hidden dimension \(128\), and the output decoder consists of a single layer with dimension \(64\). The main network contains \(6\) attention layers with embedding dimension \(128\), using \(8\) attention heads with head dimension \(16\), and a dropout rate of \(0.25\).

\paragraph{FS-Mol experiment}
For FS-Mol, we use a meta-batch size of \(16\) and train for \(100\) epochs with learning rate \(3\times10^{-4}\). 
We evaluate on the validation set after each epoch and retain the checkpoint with the lowest validation NLL.
The architecture is increased to an embedding and feed-forward dimension of \(256\), decoder dimension \(128\), \(8\) attention layers, and head dimension \(32\), while retaining \(8\) attention heads, a \(2\)-layer input embedder, and a \(1\)-layer decoder. We use at most \(5\) auxiliary datasets per task and subsample each to at most \(512\) observations, with cosine-similarity threshold $\gamma = 0.5$. For molecular representations, we use Hyformer embeddings \citep{izdebski2026synergistic}.

\paragraph{Auxiliary dataset construction}
The construction of $\mathcal{D}^{\mathrm{aux}}$ in \ourmethod{} is adapted to the notion of assay relatedness available in each experimental setting:
\begin{itemize}
    \item Controlled-heterogeneity experiment. 
    We construct $\mathcal{D}^{\mathrm{aux}}$ according to Eq.~\ref{eq:icicl_main_theorem}, selecting at most five same-target datasets, ordered by decreasing average $5$-nearest-neighbor Tanimoto similarity to the current task. The similarity is computed exclusively from molecular inputs; target outputs are never used in auxiliary-dataset selection. Auxiliary datasets containing more than $512$ observations are randomly subsampled to that size.
    \item Homologous-target transfer experiment. For the EPHB evaluation, where same-target auxiliary assays are unavailable, we instead construct $\mathcal{D}^{\mathrm{aux}}$ according to Proposition~\ref{prop:ashman_to_protfams}. For each such EPHB assay, we sample five auxiliary tasks from the pool of assays targeting any member of the EPHA subfamily and order them by decreasing activity variance. Each auxiliary dataset is again subsampled to at most $512$ observations.
    \item FS-Mol \& retrospective Bayesian optimization.
    We construct $\mathcal{D}^{\mathrm{aux}}$ using assay-description similarity, as described in Algorithm~\ref{algo}, retaining at most five auxiliary datasets per task and subsampling each to at most $512$ observations.
\end{itemize}

\subsubsection{Baselines}
\label{app:implementation-baselines}
Implementations of the baseline methods were based primarily on the code released by \citet{chen2023metalearning}, \citet{stanley2021fsmol}, and implementations of several Neural Process variants available at \url{https://github.com/cambridge-mlg/tnp}, released under the MIT License.
\begin{itemize}
    \item Transformer Neural Process (TNP) \citep{pmlr-v162-nguyen22b}. We use the same architecture and training setup as for our model. In particular, setting the number of auxiliary datasets to zero reduces our model to the standard TNP-D (using the efficient implementation described by \citet{feng2023latent}).
    \item Conditional Neural Process (CNP) \citep{pmlr-v80-garnelo18a}. For FS-Mol, we use mean aggregation and a \(256\)-dimensional context representation, with a decoder hidden dimension of \(256\) and with Hyformer embeddings as molecular representations. For the controlled-heterogeneity experiment we reduce both the context representation and decoder hidden dimensions to \(128\).
    \item Attentive Neural Process (ANP) \citep{kim2018attentive}. We match the TNP baseline in architectural parameters, including the embedding, feed-forward, decoder, and attention dimensions, as well as the number of attention heads and layers, while omitting the latent path.
    \item Adaptive Deep Kernel Fitting with Implicit Function Theorem (ADKF-IFT) \citep{chen2023metalearning}. For FS-Mol, we retain the feature extractor used in the original implementation, combining ECFP fingerprints with a GNN-MT from \citet{stanley2021fsmol}. For the controlled-heterogeneity experiment, we replace the GNN with a two-layer MLP with hidden and output dimensions of \(1024\), operating on the same \(1024\)-bit fingerprints used by the GP oracle.
    \item Deep Kernel Transfer (DKT) \citep{NEURIPS2020_b9cfe8b6}. We use the implementation provided by \citet{chen2023metalearning} and follow the same experimental setup and feature-extractor choices as for ADKF-IFT.
    \item Model-Agnostic Meta-Learning (MAML) \citep{pmlr-v70-finn17a}. We use the implementation from \citet{stanley2021fsmol} based on GNN-ST. For the controlled-heterogeneity experiment, we replace the GNN with the same two-layer MLP used for ADKF-IFT and DKT. The regression head for the kinase experiments is an MLP with hidden dimensions \(256\) and \(64\).
    \item ST-MAML \citep{pmlr-v180-wang22c}. We preserve the MAML setup, adding a stochastic task encoder with hidden dimension 128, a 32-dimensional latent representation, and a 32-dimensional task feature. We weight the KL regularization term by 0.1 and average predictions over five context-conditioned latent samples at evaluation.
    \item Frequent Hitters \citep{schimunek2023contextenriched}. Frequent Hitters is simply an MLP trained on data pooled across all training assays, without using the support set of the target assay. Following a hyperparameter sweep, we use four hidden layers of width \(16\), batch size \(256\), learning rate \(10^{-3}\), no dropout or weight decay, and train for up to \(100\) epochs with an early-stopping patience of \(10\) epochs. Molecular inputs are represented using ECFP fingerprints.
    \item Similarity Search. As an additional sanity-check baseline proposed by \citet{schimunek2023contextenriched}, we use Morgan fingerprints with radius \(2\) and \(1024\) bits, and predict each target as the Tanimoto-similarity-weighted average of the context labels.
    \item TabPFNv3 \citep{grinsztajn2026tabpfn3technicalreport}. We use Morgan fingerprints with radius \(2\) and \(1024\) bits as molecular representations, which we found to perform better than Hyformer embeddings.
    \item TabICLv2 \citep{qu2026tabiclv}. We use the same molecular representation as for TabPFNv3.
    \item ICICL-TNP \citep{pmlr-v253-ashman24a}. Across all settings, we always construct $\mathcal{D}^{\mathrm{aux}}$ in the same way as for \ourmethod{}, such that both models condition on exactly the same auxiliary data. As in \ourmethod{}, we share the projection parameters between the corresponding MHSA and MHCA blocks of the context and auxiliary branches. We do not include dataset-of-origin one-hot indicators, as auxiliary datasets are represented separately through their pseudo-token sets. We use 64 pseudo-tokens for each auxiliary dataset and set the number of context pseudo-tokens equal to the context size, $|\mathcal{D}^{C}|$.
    \item RF-$\mathcal{D}^{\mathrm{aux}}$. We construct $\mathcal{D}^{\mathrm{aux}}$ in the same way as for \ourmethod{} and then concatenate its datasets with the task-specific context set. A random forest with 250 trees is then fit to that combined dataset.
\end{itemize}

\subsection{Computational resources}
A single training run of \ourmethod{} was performed on an NVIDIA H100 GPU and required approximately 24 hours. The runtime ablation in Section ~\ref{exp:ablation} was performed on an NVIDIA Tesla V100S 32\,GB GPU with an Intel Xeon Gold 6226R CPU @ 2.90\,GHz.

\section{Extended results}
\label{app:extended_results}

\subsection{Controlled-heterogeneity experiment}
\label{app:extended_results_hetero}

\begin{figure}[H]
  \centering
\includegraphics[width=\linewidth]{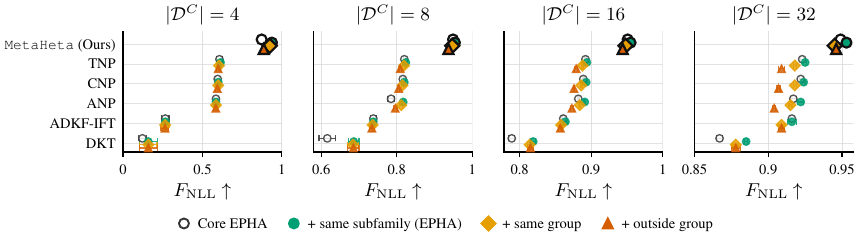}
    \caption{Effect of meta-training task heterogeneity on few-shot bioactivity prediction. Performance of all the considered methods generally deteriorates as $\mathcal D_{\mathrm{meta}}^{\mathrm{core}}$ is augmented with progressively less related tasks. \ourmethod{} remains the best-performing approach regardless of the training regime.}
  \label{fig:heterogeneity_exp_nll}
\end{figure}

\subsection{Retrospective bayesian optimization}
\label{app:extended_results_bo}

\begin{table}[H]
    \centering
    \caption{Average rank across retrospective Bayesian optimization tasks. For each assay, methods are ranked according to their final best-observed activity averaged over five random seeds; ranks are then averaged across the eight assays.}
    \label{tab:bo_rank}
    \begin{tabular}{lc}
        \toprule
        Method & Average rank \\
        \midrule
        TNP       & 4.375 \\
        ANP       & 5.375 \\
        CNP       & 5.375 \\
        ADKF-IFT  & 7.250 \\
        DKT       & 6.625 \\
        TabPFN    & 4.000 \\
        ICICL-TNP & 5.000 \\
        Random    & 6.375 \\
        GP        & 7.750 \\
        \midrule
        \ourmethod{} & \textbf{1.875} \\
        \bottomrule
    \end{tabular}
\end{table}

\subsection{Ablation}
\label{app:extended_results_ablation}

\begin{figure}[H]
  \centering
\includegraphics[width=\linewidth]{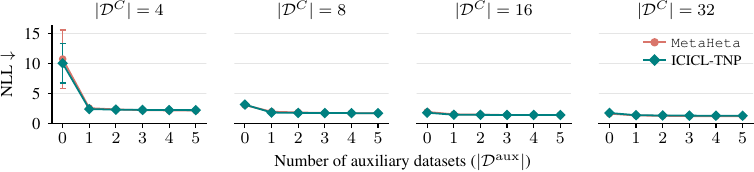}
  \caption{NLL of \ourmethod{} and ICICL-TNP with varying cardinality of $\mathcal{D}^{\mathrm{aux}}$.}
  \label{fig:ablation_nll}
\end{figure}

\section{Algorithms}
Algorithm~\ref{algo} outlines the core meta-training procedure, with $\mathcal{D}^{\mathrm{aux}}$ constructed based on assay-description similarity. For clarity, auxiliary-dataset limits, ordering, and subsampling are omitted; full experimental details are provided in Appendix~\ref{app:general_experimental_details}.

\begin{algorithm}[H]
\caption{Description-based auxiliary dataset selection and meta-training}
\label{algo}
\DontPrintSemicolon
\KwIn{
Meta-dataset $\mathcal{D}_{\mathrm{meta}}=\{\mathcal{D}_\tau\}_{\tau=1}^M$,
assay descriptions $\{a_\tau\}_{\tau=1}^M$,
meta-learner $p_\theta$,
embedding model $\phi$,
similarity threshold $\gamma$
}
\BlankLine
Compute assay-description embeddings: \linebreak
$h_\tau \leftarrow \phi(a_\tau)
\qquad \forall \tau \in \{1,\ldots,M\}$\;
\BlankLine
Compute pairwise assay similarities: \linebreak
$s_{\tau,\tau'}
\leftarrow
\frac{\langle h_\tau,h_{\tau'}\rangle}
{\|h_\tau\|\,\|h_{\tau'}\|}
\qquad
\forall \tau,\tau' \in \{1,\ldots,M\}$\;

\Repeat{convergence}{
    Sample a task $\tau$ from $\mathcal{D}_{\mathrm{meta}}$: \linebreak
    $\left((X^T, Y^T), \mathcal{D}^C\right) \sim \mathcal{D}_{\mathrm{meta}}$\;
    Initialize auxiliary data: \linebreak
    $\mathcal{D}^{\mathrm{aux}} \leftarrow \emptyset$\;
    \ForEach{$\tau' \in \{1,\ldots,M\}\setminus\{\tau\}$}{
        \If{$s_{\tau,\tau'} \geq \gamma$}{
            $\mathcal{D}^{\mathrm{aux}}
            \leftarrow
            \mathcal{D}^{\mathrm{aux}}
            \cup \{\mathcal{D}_{\tau'}\}$\;
        }
    }
    
Compute the negative log-likelihood: \linebreak
$
\mathcal{L}_{\mathrm{NLL}}
\leftarrow
-\log p_\theta
\left(
Y^T
\mid
X^T,
\mathcal{D}^C,
\mathcal{D}^{\mathrm{aux}}
\right)
$\;
Update $\theta$: \linebreak
$
\theta
\leftarrow
\theta
-
\alpha\nabla_\theta
\mathcal{L}_{\mathrm{NLL}}
$\;
}
\end{algorithm}

\section{Extended related work}
\label{app:more_rel_work}

\paragraph{Task heterogeneity in meta-learning}
Task heterogeneity in meta-learning has been addressed by prior work through two broad strategies. One line of work explicitly models multiple task regimes using mixtures of meta-learners or hierarchical task structures, with individual components specializing to subsets of related tasks \citep{wang2022learning, pmlr-v97-yao19b, NEURIPS2019_7a9a322c, pmlr-v161-zhou21a, chan2023embracingassayheterogeneityneural}. These approaches typically rely on a prespecified number of task clusters or mixture components, thereby imposing some prior structure on the task distribution. 
This is difficult to justify for bioactivity prediction, where assay heterogeneity arises from multiple interacting factors and may not admit a small, disjoint set of task regimes. 
In large assay collections the relevant regimes may be numerous and overlapping, making a global mixture cardinality difficult to specify.
Among these approaches, \citet{NEURIPS2019_7a9a322c} is an exception, using a Dirichlet process mixture that can introduce additional components as it encounters new task regimes, but still assumes a discrete mixture structure over tasks.
A complementary line of work, more directly comparable to our setting, instead infers task-specific representations from the available context set and uses them to condition or modulate the meta-learner, thereby adapting to heterogeneous task distributions without explicitly enumerating the underlying task regimes \citep{pmlr-v180-wang22c, Vettoruzzo2024, NEURIPS2019_e4da3b7f}. 
Notably, \citet{pmlr-v180-wang22c} explicitly address the difficulty of inferring task identity from scarce context by representing it as a stochastic latent variable, which captures uncertainty over the underlying task regime.

\paragraph{Comparison with ICICL-TNP \citep{pmlr-v253-ashman24a}} 
The architecture of \ourmethod{} is related to ICICL-TNP, but differs primarily in how information from multiple auxiliary datasets is represented and exchanged.
Both architectures use coupled TNP-style processing paths for the task-specific context and auxiliary data, while avoiding quadratic attention over the potentially large auxiliary datasets. ICICL-TNP achieves this using a Perceiver-style \citep{pmlr-v139-jaegle21a} pseudo-token architecture: (i) the task-specific context set and each auxiliary dataset are separately encoded into their own sets of pseudo-tokens; (ii) the auxiliary pseudo-tokens attend to the context pseudo-tokens; (iii) self-attention is applied separately within each set of pseudo-tokens; and (iv) the context pseudo-tokens attend back to the auxiliary pseudo-tokens before conditioning the target representations. Consequently, observations from different auxiliary datasets exchange information only indirectly. \ourmethod{} concatenates auxiliary observations into a single sequence and applies linear self-attention across them, enabling direct information exchange between observations from different auxiliary datasets. This is made possible by augmenting each observation with a one-hot indicator of its task of origin, which preserves task identity after concatenation, while the use of FAVOR+ \citep{choromanski2021rethinking} in the auxiliary branch keeps self-attention computationally efficient even for large auxiliary sets. This removes the pseudo-token bottleneck and, in our setting, enables substantially more efficient conditioning on large auxiliary collections. However, as a result of representing task identity explicitly through one-hot indicators, the input embedding dimension fixes the maximum number of auxiliary datasets that can be accommodated without modifying the architecture. ICICL-TNP does not have this restriction, since task identity is represented structurally through separate pseudo-token sets.

\section{Discussion}
\label{app:discussion}
\subsection{Limitations}
\paragraph{Scope of the theoretical guarantees}
Both the result of \citet{pmlr-v253-ashman24a} in Eq.~\ref{eq:icicl_main_theorem} and Proposition~\ref{prop:ashman_to_protfams}, which motivate our selection of auxiliary datasets, are derived for exact Bayesian posterior predictive distributions under their respective generative assumptions.
They therefore do not directly guarantee improvements for finite-capacity models trained on a finite meta-dataset. Likewise, description-based relatedness is only a heuristic for identifying useful auxiliary assays.
However, we demonstrate that auxiliary conditioning remains effective across all three settings considered in this work: when relatedness is defined by shared target identity, by homology between distinct targets, and heuristically from assay descriptions. In the latter setting, randomly shuffling the retrieved auxiliary assays (Section~\ref{exp:ablation}) substantially degrades predictive performance, indicating that the observed gains depend on meaningful assay relatedness.

\paragraph{Fixed auxiliary-dataset capacity} 
Because \ourmethod{} encodes dataset identity using one-hot indicators, the maximum number of auxiliary datasets is fixed by the input dimensionality and cannot be increased after training without modifying the model, as elaborated in Appendix~\ref{app:more_rel_work}. In practice, however, we found this limitation to be mild: the model can simply be trained with a conservative upper bound and evaluated with fewer auxiliary datasets. Our ablation in Section~\ref{exp:ablation} demonstrates that performance remains strong across varying inference-time auxiliary-set cardinalities.

\subsection{Future work}
Proposition~\ref{prop:ashman_to_protfams} considers a two-level hierarchy in which protein-specific components share a common family-level component. In practice, biological relatedness can exhibit richer structure. A natural extension would therefore be to incorporate multiple levels of relatedness, allowing auxiliary information to be shared across increasingly broad groups of related targets. Beyond hierarchical structure, future work could also incorporate additional data modalities, such as protein sequence or structure and richer assay metadata, to define relatedness more comprehensively.

\section{Attention operators}
\label{app:attention-operators}

\paragraph{MHSA and MHCA \citep{NIPS2017_3f5ee243}}
Let $Z\in\mathbb R^{N\times d_z}$ and
$Z'\in\mathbb R^{M\times d_z}$ denote query and conditioning
representations, respectively, where $d_z$ is the embedding
dimension. Let $H$ denote the number of attention heads,
and let $d_k$ and $d_v$ denote the query/key and value
dimensions within each head. Layer indices are omitted
for clarity.

For each head $h=1,\ldots,H$, the query, key, and value
matrices are defined as
\begin{equation}
Q_h=ZW_{Q,h},\qquad
K_h=Z'W_{K,h},\qquad
V_h=Z'W_{V,h},
\end{equation}
where $W_{Q,h},W_{K,h}\in\mathbb R^{d_z\times d_k}$
and $W_{V,h}\in\mathbb R^{d_z\times d_v}$.
Multi-head cross-attention is defined as
\begin{align}
A_h(Z,Z')
&=\operatorname{softmax}
\left(\frac{Q_hK_h^\top}{\sqrt{d_k}}\right)V_h,\\
\operatorname{MHCA}(Z,Z')
&=\operatorname{Concat}
\bigl(A_1(Z,Z'),\ldots,A_H(Z,Z')\bigr)W_O,
\end{align}
where softmax is applied rowwise, concatenation is along
the feature dimension, and $W_O\in\mathbb R^{Hd_v\times d_z}$.
Multi-head self-attention is defined as
\begin{equation}
\operatorname{MHSA}(Z)=\operatorname{MHCA}(Z,Z).
\end{equation}

\paragraph{FAVOR+ \citep{choromanski2021rethinking}}
The positive random feature map
$\phi:\mathbb R^{d_k}\to\mathbb R^r_{>0}$, where $r$ is
the number of random features, is defined as
\begin{equation}
\phi(u)
=
\frac{\exp(-\|u\|^2/2)}{\sqrt r}
\begin{bmatrix}
\exp(\omega_1^\top u)\\
\vdots\\
\exp(\omega_r^\top u)
\end{bmatrix},
\end{equation}
where the projection vectors have marginal distribution
$\omega_i\sim\mathcal N(0,I_{d_k})$ and are sampled using
orthogonal random features, in independent blocks when
$r>d_k$. This yields
\begin{equation}
\mathbb E\!\left[\phi(u)^\top\phi(v)\right]
=\exp(u^\top v).
\end{equation}
Writing $\Phi$ for the rowwise application of $\phi$,
the feature matrices are defined as
\begin{equation}
Q'_h=\Phi(d_k^{-1/4}Q_h)\in\mathbb R^{N\times r},
\qquad
K'_h=\Phi(d_k^{-1/4}K_h)\in\mathbb R^{M\times r}.
\end{equation}
Bidirectional FAVOR+ approximates each attention head by
\begin{equation}
\widehat A_h(Z,Z')
=
D_h^{-1}Q'_h\bigl((K'_h)^\top V_h\bigr),
\qquad
D_h=\operatorname{diag}
\left(Q'_h\bigl((K'_h)^\top\mathbf 1_M\bigr)\right),
\end{equation}
where $\mathbf 1_M$ is the length-$M$ vector of ones.
By $\operatorname{FAVOR+}(Z)$, we denote the corresponding
multi-head self-attention operation:
\begin{equation}
\operatorname{FAVOR+}(Z)
=
\operatorname{Concat}
\bigl(\widehat A_1(Z,Z),\ldots,\widehat A_H(Z,Z)\bigr)W_O.
\end{equation}

\section{Proofs}
\label{app:proofs}

\subsection{Permutation invariance}
\begin{proposition}[Permutation invariance]
\label{prop:perm_invariance}
    Let $\mathcal D^{\mathrm{aux}}=\{\mathcal D_j\}_{j=1}^K$,
    and let $S:\mathcal D^K\to\mathcal D^K$ be a map
    satisfying
    \begin{equation}
    S(\mathcal D_{\rho(1)},\ldots,\mathcal D_{\rho(K)})
    =
    S(\mathcal D_1,\ldots,\mathcal D_K)
    \qquad
    \text{for every permutation }\rho.
    \label{eq:canonical-ordering}
    \end{equation}
    With the dataset indicator $e_j$ attached to each observation
    in the $j$-th dataset after sorting, predictions are invariant
    to permutations of datapoints within each
    $\mathcal D_j\in\mathcal D^{\mathrm{aux}}$ and to permutations
    of datasets in $\mathcal D^{\mathrm{aux}}$, at inference
    with fixed FAVOR+ features.    
\end{proposition}

\paragraph{Proof.}
For fixed random features, $\Phi(PQ)=P\Phi(Q)$ for any
permutation matrix $P$, and
\begin{equation}
(PK')^\top(PV)=K'^\top V,
\qquad
(PK')^\top\mathbf 1=K'^\top\mathbf 1.
\end{equation}
Thus bidirectional FAVOR+ is permutation equivariant.
Permutation-equivariant self-attention and cross-attention
invariant to key/value ordering yield predictions invariant
to observation order for fixed dataset indicators \citep{pmlr-v97-lee19d}.
It therefore remains to consider dataset order in
$\mathcal D^{\mathrm{aux}}$.
By assumption, $S$ produces the same canonical ordering
under any dataset permutation, preserving each dataset's
indicator and hence the predictive distribution.
\hfill$\square$

\subsection{Proposition~\ref{prop:ashman_to_protfams}}
\label{app:prop_1_proof}
For convenience, we restate the proposition here.
{\begin{propositionrestated}[{\ref{prop:ashman_to_protfams}}]
Consider an extension of the data-generating process in Eq.~\ref{eq_everything_everywhere_all_at_once}
in which:
\[
\xi_{\mathrm{family}}\sim p(\xi),
\qquad
\xi_{\mathrm{prot}},\{\xi_{\mathrm{prot}'}\}
\sim p(\cdot\mid\xi_{\mathrm{family}}),
\]
with $f_\tau\sim P_{\xi_{\mathrm{prot}}}(f)$ and
$f_{\tau'}\sim P_{\xi_{\mathrm{prot}'}}(f)$.
Assume that auxiliary data from target homologs provide no
additional information for predicting $y$ once the target's
protein-specific component is known:
\[
p_{_{\xi_{\mathrm{prot}}}}(y\mid x,\mathcal{D}^C,\mathcal{D}^{\mathrm{aux}})
=
p_{_{\xi_{\mathrm{prot}}}}(y\mid x,\mathcal{D}^C).
\]
Then:
\[
\begin{aligned}
&\mathbb E_{\xi_{\mathrm{family}},\xi_{\mathrm{prot}},
             \mathcal{D}^C,\mathcal{D}^{\mathrm{aux}}}
\left[
\mathrm{KL}\!\left(
p_{_{\xi_{\mathrm{prot}}}}(y\mid x,\mathcal{D}^C)
\,\Vert\,p(y\mid x,\mathcal{D}^C,\mathcal{D}^{\mathrm{aux}})
\right)
\right]
\\
&\qquad\leq
\mathbb E_{\xi_{\mathrm{family}},\xi_{\mathrm{prot}},\mathcal{D}^C}
\left[
\mathrm{KL}\!\left(
p_{_{\xi_{\mathrm{prot}}}}(y\mid x,\mathcal{D}^C)
\,\Vert\,p(y\mid x,\mathcal{D}^C)
\right)
\right].
\end{aligned}
\]
\end{propositionrestated}}

\noindent\textbf{Proof.}
The integrands on both sides do not depend explicitly on
$\xi_{\mathrm{family}}$, so this variable can be integrated out:
\begin{equation}
\begin{aligned}
&\mathbb E_{\xi_{\mathrm{prot}},\mathcal{D}^C,\mathcal{D}^{\mathrm{aux}}}
\left[
\mathrm{KL}\!\left(
p_{_{\xi_{\mathrm{prot}}}}(y\mid x,\mathcal{D}^C)
\,\Vert\,p(y\mid x,\mathcal{D}^C,\mathcal{D}^{\mathrm{aux}})
\right)
\right]
\\
&\qquad\leq
\mathbb E_{\xi_{\mathrm{prot}},\mathcal{D}^C}
\left[
\mathrm{KL}\!\left(
p_{_{\xi_{\mathrm{prot}}}}(y\mid x,\mathcal{D}^C)
\,\Vert\,p(y\mid x,\mathcal{D}^C)
\right)
\right].
\end{aligned}
\end{equation}
The rest of the argument follows the proof of
\citet{pmlr-v253-ashman24a}, with all expectations taken under the
marginal distributions induced by the hierarchical generative
process.

\medskip
\noindent\textbf{LHS.}
Expanding the KL divergence gives
\begin{equation}
\begin{aligned}
\mathrm{LHS}
={}&-\mathbb E_{\xi_{\mathrm{prot}},\mathcal{D}^C,\mathcal{D}^{\mathrm{aux}}}
H[p_{_{\xi_{\mathrm{prot}}}}(y\mid x,\mathcal{D}^C)]
\\
&-\mathbb E_{\xi_{\mathrm{prot}},\mathcal{D}^C,\mathcal{D}^{\mathrm{aux}}}
\int p_{_{\xi_{\mathrm{prot}}}}(y\mid x,\mathcal{D}^C)
\log p(y\mid x,\mathcal{D}^C,\mathcal{D}^{\mathrm{aux}})\,dy.
\end{aligned}
\end{equation}
The first integrand does not depend on $\mathcal{D}^{\mathrm{aux}}$,
so this variable also can be integrated out.
For the second term, conditional independence and
marginalization yield
\begin{equation}
\begin{aligned}
&-\mathbb E_{\xi_{\mathrm{prot}},\mathcal{D}^C,\mathcal{D}^{\mathrm{aux}}}
\int p_{_{\xi_{\mathrm{prot}}}}(y\mid x,\mathcal{D}^C)
\log p(y\mid x,\mathcal{D}^C,\mathcal{D}^{\mathrm{aux}})\,dy
\\
&\quad=
-\mathbb E_{\mathcal{D}^C,\mathcal{D}^{\mathrm{aux}}}\int
\mathbb E_{\xi_{\mathrm{prot}}\mid\mathcal{D}^C,\mathcal{D}^{\mathrm{aux}}}
\left[
p_{_{\xi_{\mathrm{prot}}}}(y\mid x,\mathcal{D}^C,\mathcal{D}^{\mathrm{aux}})
\right]
\log p(y\mid x,\mathcal{D}^C,\mathcal{D}^{\mathrm{aux}})\,dy
\\
&\quad=
-\mathbb E_{\mathcal{D}^C,\mathcal{D}^{\mathrm{aux}}}\int
p(y\mid x,\mathcal{D}^C,\mathcal{D}^{\mathrm{aux}})
\log p(y\mid x,\mathcal{D}^C,\mathcal{D}^{\mathrm{aux}})\,dy
\\
&\quad=
\mathbb E_{\mathcal{D}^C,\mathcal{D}^{\mathrm{aux}}}
H[p(y\mid x,\mathcal{D}^C,\mathcal{D}^{\mathrm{aux}})].
\end{aligned}
\end{equation}
Hence
\begin{equation}
\begin{aligned}
\mathrm{LHS}
={}&-\mathbb E_{\xi_{\mathrm{prot}},\mathcal{D}^C}
H[p_{_{\xi_{\mathrm{prot}}}}(y\mid x,\mathcal{D}^C)]
\\
&+\mathbb E_{\mathcal{D}^C,\mathcal{D}^{\mathrm{aux}}}
H[p(y\mid x,\mathcal{D}^C,\mathcal{D}^{\mathrm{aux}})].
\end{aligned}
\end{equation}

\medskip
\noindent\textbf{RHS.}
Similarly,
\begin{equation}
\begin{aligned}
\mathrm{RHS}
={}&-\mathbb E_{\xi_{\mathrm{prot}},\mathcal{D}^C}
H[p_{_{\xi_{\mathrm{prot}}}}(y\mid x,\mathcal{D}^C)]
\\
&-\mathbb E_{\xi_{\mathrm{prot}},\mathcal{D}^C}
\int p_{_{\xi_{\mathrm{prot}}}}(y\mid x,\mathcal{D}^C)
\log p(y\mid x,\mathcal{D}^C)\,dy.
\end{aligned}
\end{equation}
Marginalizing $\xi_{\mathrm{prot}}$ in the second term gives
\begin{equation}
\begin{aligned}
&-\mathbb E_{\xi_{\mathrm{prot}},\mathcal{D}^C}
\int p_{_{\xi_{\mathrm{prot}}}}(y\mid x,\mathcal{D}^C)
\log p(y\mid x,\mathcal{D}^C)\,dy
\\
&\quad=
-\mathbb E_{\mathcal{D}^C}\int
\mathbb E_{\xi_{\mathrm{prot}}\mid\mathcal{D}^C}
\left[p_{_{\xi_{\mathrm{prot}}}}(y\mid x,\mathcal{D}^C)\right]
\log p(y\mid x,\mathcal{D}^C)\,dy
\\
&\quad=
-\mathbb E_{\mathcal{D}^C}\int
p(y\mid x,\mathcal{D}^C)\log p(y\mid x,\mathcal{D}^C)\,dy
\\
&\quad=
\mathbb E_{\mathcal{D}^C}H[p(y\mid x,\mathcal{D}^C)].
\end{aligned}
\end{equation}
Thus
\begin{equation}
\begin{aligned}
\mathrm{RHS}
={}&-\mathbb E_{\xi_{\mathrm{prot}},\mathcal{D}^C}
H[p_{_{\xi_{\mathrm{prot}}}}(y\mid x,\mathcal{D}^C)]
\\
&+\mathbb E_{\mathcal{D}^C}H[p(y\mid x,\mathcal{D}^C)].
\end{aligned}
\end{equation}

Finally, conditioning does not increase entropy in expectation:
\begin{equation}
\mathbb E_{\mathcal{D}^C,\mathcal{D}^{\mathrm{aux}}}
H[p(y\mid x,\mathcal{D}^C,\mathcal{D}^{\mathrm{aux}})]
\leq
\mathbb E_{\mathcal{D}^C}H[p(y\mid x,\mathcal{D}^C)].
\end{equation}
The remaining terms are identical, so
$\mathrm{LHS}\leq\mathrm{RHS}$. \hfill$\square$

\end{document}